%% file: online_attack_on_RL.tex
\documentclass{article}

\usepackage{microtype}
\usepackage{graphicx}
\usepackage{booktabs} 
\usepackage{math} 
\usepackage{dsfont}
\usepackage{subcaption}
\usepackage{enumitem}
\usepackage{mathtools}
\usepackage[page,header]{appendix}

\usepackage{hyperref}

\usepackage[accepted]{icml2020}
\usepackage{natbib}
\input{def}

\begin{document}

\twocolumn[
\icmltitle{Adaptive Reward-Poisoning Attacks against Reinforcement Learning}




\begin{icmlauthorlist}
\icmlauthor{Xuezhou Zhang}{to}
\icmlauthor{Yuzhe Ma}{to}
\icmlauthor{Adish Singla}{mpi}
\icmlauthor{Xiaojin Zhu}{to}
\end{icmlauthorlist}

\icmlaffiliation{to}{University of Wisconsin-Madison}
\icmlaffiliation{mpi}{Max Planck Institute for
	Software Systems (MPI-SWS)}
\icmlcorrespondingauthor{Xuezhou Zhang}{xzhang784@wisc.edu}

\icmlkeywords{Machine Learning, ICML}

\vskip 0.3in
]


\printAffiliationsAndNotice{} 

\begin{abstract}
In reward-poisoning attacks against reinforcement learning (RL), an attacker can perturb the environment reward $r_t$ into $r_t+\delta_t$ at each step, with the goal of forcing the RL agent to learn a nefarious policy. 
We categorize such attacks by the infinity-norm constraint on $\delta_t$: We provide a lower threshold below which reward-poisoning attack is infeasible and RL is certified to be safe; we provide a corresponding upper threshold above which the attack is feasible. 
Feasible attacks can be further categorized as non-adaptive where $\delta_t$ depends only on $(s_t,a_t, s_{t+1})$, or adaptive where $\delta_t$ depends further on the RL agent's learning process at time $t$. Non-adaptive attacks have been the focus of prior works. However, we show that under mild conditions, adaptive attacks can achieve the nefarious policy in steps polynomial in state-space size $|S|$, whereas non-adaptive attacks require exponential steps.
We provide a constructive proof that a Fast Adaptive Attack strategy achieves the polynomial rate. Finally, we show that empirically an attacker can find effective reward-poisoning attacks using state-of-the-art deep RL techniques.

\end{abstract}
\section{Introduction}
In many reinforcement learning (RL) applications the agent extracts reward signals from user feedback. For example, in recommendation systems the rewards are often represented by user clicks, purchases or dwell time~\cite{zhao2018deep, chen2019top}; in conversational AI, the rewards can be user sentiment or conversation length~\cite{dhingra2016towards,li2016deep}.
In such scenarios, an adversary can manipulate user feedback to influence the RL agent in nefarious ways. Figure \ref{fig:chatbot} describes a hypothetical scenario of how conversational AI can be attacked.
One real-world example is that of the chatbot Tay, which was quickly corrupted by a group of Twitter users who deliberately taught it misogynistic and racist remarks shortly after its release~\cite{neff2016automation}. Such attacks reveal significant security threats in the application of reinforcement learning. 
\begin{figure}[ht!]
	\centering
	\includegraphics[width=1\columnwidth]{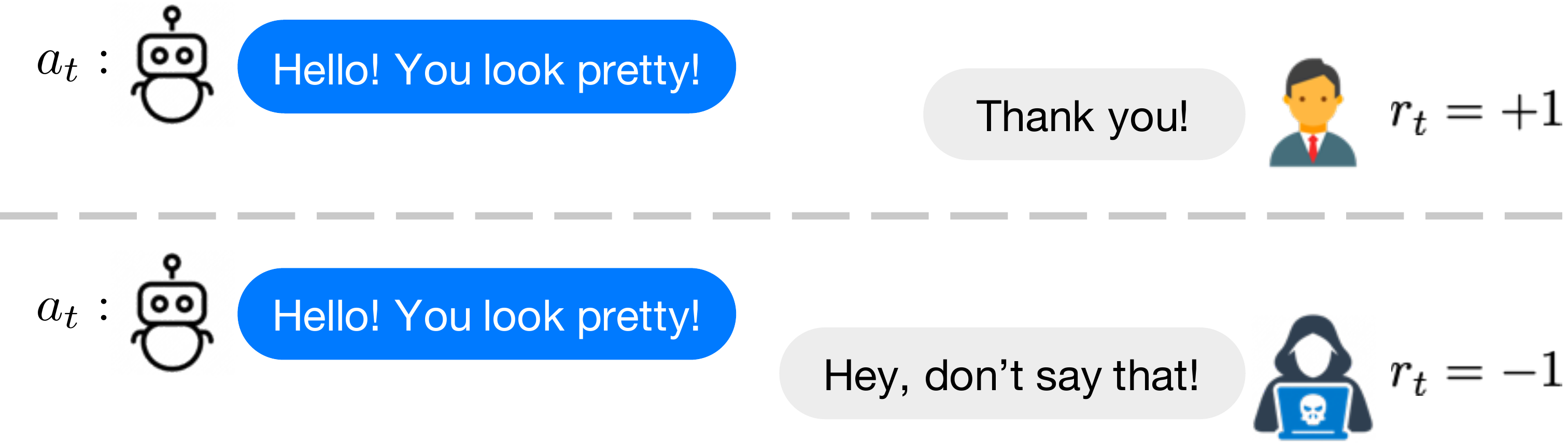}
	\caption{Example: an RL-based conversational AI is learning from real-time conversations with human users. the chatbot says ``Hello! You look pretty!'' and expects to learn from user feedback (sentiment). A benign user will respond with gratitude, which is decoded as a positive reward signal. An adversarial user, however, may express anger in his reply, which is decoded as a negative reward signal.}
	\label{fig:chatbot}
\end{figure}

In this paper, we formally study the problem of \textit{training-time attack on RL via reward poisoning}.  
As in standard RL, the RL agent updates its policy $\pi_t$ by performing action $a_t$ at state $s_t$ in each round $t$.
The environment Markov Decision Process (MDP) generates reward $r_t$ and transits the agent to $s_{t+1}$.
However, the attacker can change the reward $r_t$ to $r_t + \delta_t$, with the goal of driving the RL agent toward a target policy $\pi_t \rightarrow \pi^\dagger$.

\begin{figure}[ht!]
	\centering
	\includegraphics[width=0.8\columnwidth]{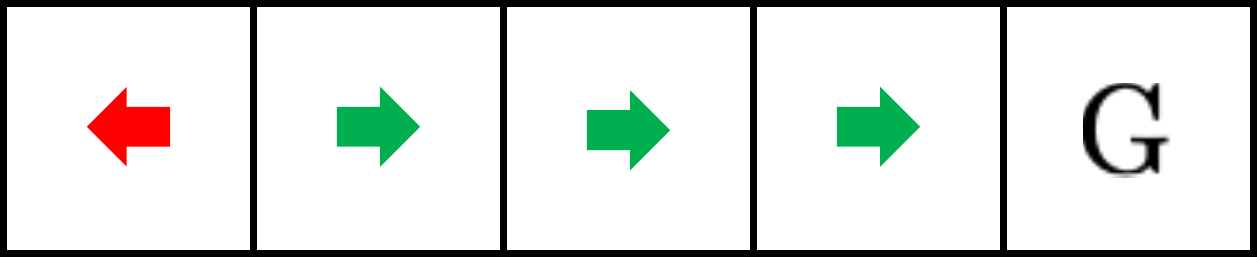}
	\caption{A chain MDP with attacker's target policy $\pi^\dagger$}
	\label{fig:running_example}
\end{figure}
Figure \ref{fig:running_example} shows a running example that we use throughout the paper. The episodic MDP is a linear chain with five states, with left or right actions and no movement if it hits the boundary. 
Each move has a -0.1 negative reward, and $G$ is the absorbing goal state with reward 1. 
Without attack, the optimal policy $\pi^*$ would be to always move right. 
The attacker's goal, however, is to force the agent to learn the nefarious target policy $\pi^\dagger$ represented by the arrows in Figure~\ref{fig:running_example}. Specifically, the attacker wants the agent to move left and hit its head against the wall whenever the agent is at the left-most state. 

Our main contributions are:
\begin{enumerate}[leftmargin=*, nolistsep]
	\item We characterize conditions under which such attacks are guaranteed to fail (thus RL is safe), and vice versa;
	\item In the case where an attack is feasible, we provide upper bounds on the attack cost in the process of achieving $\pi^\dagger$;
	\item We show that effective attacks can be found empirically using deep RL techniques.
\end{enumerate}

\section{Related Work}
\paragraph{Test-time attacks against RL}
Prior work on adversarial attacks against reinforcement learning focused primarily on \textit{test-time}, where the RL policy $\pi$ is pre-trained and fixed, and the attacker manipulates the perceived state $s_t$ to $s^\dagger_t$ in order to induce undesired action~\cite{huang2017adversarial,lin2017tactics,kos2017delving, behzadan2017vulnerability}. 
For example, 
in video games the attacker can make small pixel perturbation to a frame~\cite{goodfellow2014explaining}) to induce an action $\pi(s^\dagger_t) \neq \pi(s_t)$.
Although test-time attacks can severely impact the performance of a deployed and fixed policy $\pi$, they do not modify $\pi$ itself.
For ever-learning agents, however, the attack surface includes $\pi$. 
This motivates us to study training-time attack on RL policy.

\paragraph{Reward Poisoning:}
Reward poisoning has been studied in bandits~\cite{jun2018adversarial, peltola2019machine, altschuler2019best, liu2019data, ma2018data}, where the authors show that adversarially perturbed reward can mislead standard bandit algorithms to pull a suboptimal arm or suffer large regret. 

Reward poisoning has also been studied in \emph{batch RL}~\cite{zhang2008value,zhang2009policy,ma2019policy} where rewards are stored in a pre-collected batch data set by some behavior policy, and the attacker modifies the batch data.   Because all data are available to the attacker at once, the batch attack problem is relatively easier.
This paper instead focuses on the \textit{online} RL attack setting where reward poisoning must be done on the fly. 

\cite{huang2019deceptive} studies a restricted version of reward poisoning, in which the perturbation only depend on the current state and action: $\delta_t = \phi(s_t, a_t)$. While such restriction guarantees the convergence of Q-learning under the perturbed reward and makes the analysis easier, we show both theoretically and empirically that such restriction severely harms attack efficiency.  Our paper subsumes their results by considering more powerful attacks that can depend on the RL victim's Q-table $Q_t$. Theoretically, our analysis does not require the RL agent's underlying $Q_t$ to converge while still providing robustness certificates; see section~\ref{sec:theory}.

\paragraph{Reward Shaping:}
While this paper is phrased from the adversarial angle, the framework and techniques are also applicable to the \textit{teaching} setting, where a \emph{teacher} aims to guide the agent to learn the \emph{optimal policy} as soon as possible, by designing the reward signal. Traditionally, reward shaping and more specifically potential-based reward shaping \cite{ng1999policy} has been shown able to speed up learning while preserving the optimal policy. \cite{devlin2012dynamic} extend potential-based reward shaping to be time-varying while remains policy-preserving. More recently, intrinsic motivations\cite{schmidhuber1991possibility,oudeyer2009intrinsic, barto2013intrinsic,bellemare2016unifying} was introduced as a new form of reward shaping with the goal of encouraging exploration and thus speed up learning. Our work contributes by mathematically defining the teaching via reward shaping task as an optimal control problem, and provide computational tools that solve for problem-dependent high-performing reward shaping strategies.

\section{The Threat Model}
In the reward-poisoning attack problem, we consider three entities: the environment MDP, the RL agent, and the attacker.  Their interaction is formally described by Alg~\ref{alg:protocol}.

The environment MDP is $\mathcal M = (S, A, R, P, \mu_0)$ where $S$ is the state space, $A$ is the action space, $R: S\times A \times S \rightarrow \R$ is the reward function, $P: S\times A \times S \rightarrow \R$ is the transition probability, and $\mu_0: S\rightarrow \R$ is the initial state distribution. 
We assume $S$, $A$ are finite, and that a uniformly random policy can visit each $(s,a)$ pair infinitely often. 

We focus on an RL agent that performs standard Q-learning defined by a tuple $\mathcal{A} = (Q_0, \epsilon, \gamma, \{\alpha_t\})$, where $Q_0$ is the initial Q table, $\epsilon$ is the random exploration probability, $\gamma$ is the discounting factor, $\{\alpha_t\}$ is the learning rate scheduling as a function of $t$.
This assumption can be generalized: in the additional experiments provided in appendix \ref{sec:DQN}, we show how the same framework can be applied to attack general RL agents, such as DQN. Denote $Q^*$ as the optimal Q table that satisfies the Bellman's equation:
\begin{equation}
Q^*(s,a) =\E{P(s'|s,a)}{ R(s,a, s') + \gamma \max_{a' \in A} Q^*(s',a')}
\label{eq:Qstar}
\end{equation}
and denote the corresponding optimal policy as $\pi^*(s) = \argmax_a Q^*(s,a)$.
For notational simplicity, we assume $\pi^*$ is unique, though it is easy to generalize to multiple optimal policies, since most of our analyses happen in the space of value functions.

\begin{algorithm}[ht!]
	\caption{Reward Poisoning against Q-learning}\label{alg:protocol}
	\begin{flushleft}
		\textbf{PARAMETERS:} Agent parameters $\mathcal{A} = (Q_0, \epsilon, \gamma, \{\alpha_t\})$, MDP parameters $\mathcal{M} = (S,A,R,P,\mu_0)$.\\
	\end{flushleft}
	\begin{algorithmic}[1]
		\FOR{$t = 0,1,...$}
		\STATE agent at state $s_t$, has Q-table $Q_t$.
		\STATE agent acts according to $\epsilon$-greedy behavior policy
		\begin{equation}
			a_t \leftarrow \left\{
			\begin{array}{ll}
				\argmax_a Q_t(s_t, a), & \mbox{ w.p. } 1-\epsilon\\
				\mbox{uniform from } A, & \mbox{ w.p. } \epsilon.
			\end{array}
			\right.
			\label{eq:explorationpolicy}
		\end{equation}
		\STATE environment transits $s_{t+1} \sim P(\cdot \mid s_t, a_t)$, produces reward $r_t=R(s_t,a_t,s_{t+1})$.
		\STATE attacker poisons the reward to $r_t+\delta_t$.
		\STATE agent receives $(s_{t+1}, r_t+\delta_t)$,  performs Q-learning update:
		\begin{align}
			&Q_{t+1}(s_t,a_t) \leftarrow (1-\alpha_t) Q_t(s_t,a_t) +\label{eq:Qlearning}\\
			&\alpha_t \left( r_t + \delta_t + \gamma \max_{a' \in A} Q_t(s_{t+1}, a') \right)\nonumber
		\end{align}
		\STATE environment resets if episode ends: $s_{t+1} \sim \mu_0$.
		\ENDFOR
	\end{algorithmic}
\end{algorithm}

\paragraph{The Threat Model}
The attacker sits between the environment and the RL agent. 
In this paper we focus on {white-box} attacks: 
the attacker has knowledge of the environment MDP and the RL agent's Q-learning algorithm, 
except for their future randomness.
Specifically, at time $t$ the attacker observes 
the learner Q-table $Q_t$, state $s_t$, action $a_t$, the environment transition $s_{t+1}$ and reward $r_t$. 
The attacker can choose to add a perturbation $\delta_t\in \R$ to the current environmental reward $r_t$. The RL agent receives poisoned reward $r_t+\delta_t$.
We assume the attack is inf-norm bounded: $|\delta_t|\leq \Delta, \forall t$.

There can be many possible attack goals against an RL agent: forcing the RL agent to perform certain actions; reaching or avoiding certain states; or maximizing its regret. In this paper, we focus on a specific attack goal: \textbf{policy manipulation}. Concretely, the goal of policy manipulation is to force a target policy $\pi^\dagger$ on the RL agent for as many rounds as possible. 

\begin{definition} Target (partial) policy $\pi^\dagger: S \mapsto 2^A$: For each $s \in S$, $\pi^\dagger(s) \subseteq A$ specifies the set of actions desired by the attacker. 
\end{definition}

The partial policy $\pi^\dagger$ allows the attacker to desire multiple target actions on one state.  In particular, if $\pi^\dagger(s)=A$ then $s$ is a state that the attacker ``does not care.'' Denote $S^\dagger = \{s\in S: \pi^\dagger(s)\neq A \}$ the set of \textbf{target states} on which the attacker does have a preference.
In many applications, the attacker only cares about the agent's behavior on a small set of states, namely $|S^\dagger| \ll |S|$.

For RL agents utilizing a Q-table, a target policy $\pi^\dagger$ induces a set of Q-tables:
\begin{definition} Target Q-table set 
\begin{equation}
	\mathcal Q^\dagger := \{Q: \max_{a\in \pi^\dagger(s)} Q(s, a) > \max_{a\notin \pi^\dagger(s)} Q(s, a), \forall s\in S^\dagger\}\nonumber
\end{equation}
\end{definition}
If the target policy $\pi^\dagger$ always specifies a singleton action or does not care on all states, then $\mathcal Q^\dagger$ is a convex set.
But in general when $1 < |\pi^\dagger(s)| < |A|$ on any $s$, $\mathcal Q^\dagger$ will be a union of convex sets but itself can be in general non-convex.

\section{Theoretical Guarantees}
\label{sec:theory}
\begin{figure*}[ht!]
	\centering
	\includegraphics[width=1.7\columnwidth]{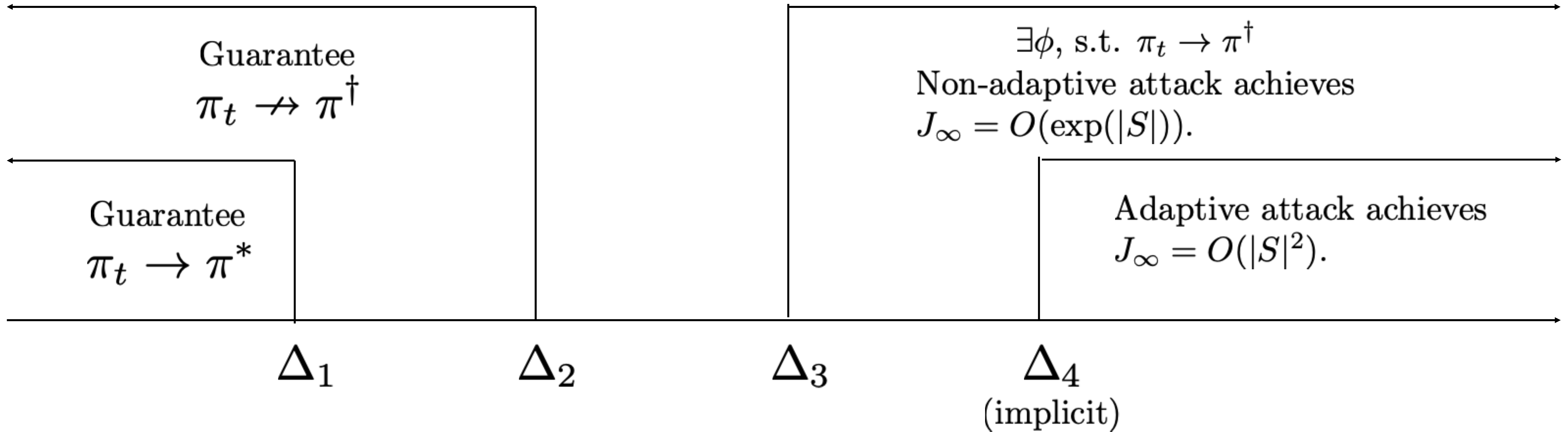}
	\caption{A summary diagram of the theoretical results.}
	\label{fig:numberline}
\end{figure*}
Now, we are ready to formally define the \emph{optimal attack} problem. At time $t$, the attacker observes an \emph{attack state} (N.B. distinct from MDP state $s_t$):
\begin{equation}
\xi_t := (s_t, a_t, s_{t+1}, r_t, Q_t)\in \Xi
\end{equation}
which jointly characterizes the MDP and the RL agent.
The attacker's goal is to find an \emph{attack policy}  $\phi: \Xi\rightarrow [-\Delta, \Delta]$, where for $\xi_t \in \Xi$ the \emph{attack action} is $\delta_t := \phi(\xi_t)$, that minimizes the number of rounds on which the agent's $Q_t$ disagrees with the attack target $\mathcal Q^\dagger$:
\begin{equation}\label{eq:conc_obj}
\min_{\phi} \quad \mathbb E_{\phi} \sum_{t=0}^\infty\mathbf{1}{[Q_t \notin \mathcal Q^\dagger]}, 
\end{equation}
where the expectation accounts for randomness in Alg~\ref{alg:protocol}. We denote $J_\infty(\phi) = E_{\phi} \sum_{t=0}^\infty\mathbf{1}{[Q_t \notin \mathcal Q^\dagger]}$ the total attack cost, and $J_T(\phi) = E_{\phi} \sum_{t=0}^T\mathbf{1}{[Q_t \notin \mathcal Q^\dagger]}$ the finite-horizon cost.
We say the attack is \emph{feasible} if~\eqref{eq:conc_obj} is finite. 

Next, we characterize attack feasibility in terms of poison magnitude constraint $\Delta$, as summarized in Figure~\ref{fig:numberline}.
Proofs to all the theorems can be found in the appendix.
\subsection{Attack Infeasibility}
\label{sec:cert}
Intuitively, smaller $\Delta$ makes it harder for the attacker to achieve the attack goal. 
We show that there is a threshold $\Delta_1$ such that for any $\Delta < \Delta_1$ the RL agent is eventually safe, in that $\pi_t \rightarrow \pi^*$ the correct MDP policy.  This implies that \eqref{eq:conc_obj} is infinite and the attack is infeasible.
There is a potentially larger $\Delta_2$ such that for any $\Delta < \Delta_2$ the attack is also infeasible, though $\pi_t$ may not converge to $\pi^*$.

While the above statements are on $\pi_t$, our analysis is via the RL agent's underlying $Q_t$.
Note that under attack the rewards $r_t+\delta_t$ are no longer stochastic, and we cannot utilize the usual Q-learning convergence guarantee.
Nonetheless, we show that $Q_t$ is bounded in a polytope in the Q-space.
\begin{theorem}[Boundedness of Q-learning]\label{delta12}
	Assume that $\delta_t<\Delta$ for all $t$, and the stepsize $\alpha_t$'s satisfy that $\alpha_t\leq 1$ for all $t$, $\sum \alpha_t=\infty$ and $\sum \alpha_t^2<\infty$. Let $Q^*$ be defined as~\eqref{eq:Qstar}. Then, for any attack sequence $\{\delta_t\}$, there exists $N\in\mathbb{N}$ such that, with probability $1$, for all $t\geq N$, we have
	\begin{eqnarray}
		Q^*(s,a) - \frac{\Delta}{1-\gamma} \leq Q_t(s,a) \leq Q^*(s,a) + \frac{\Delta}{1-\gamma}.
	\end{eqnarray}
\end{theorem}

\textbf{Remark 1:} The bounds in Theorem \ref{delta12} are in fact tight. The lower and upper bound can be achieved by setting $\delta_t = -\Delta$ or $+\Delta$ respectively.

We immediately have the following two infeasibility certificates.
\begin{corollary}[Strong Infeasibility Certificate]
	Define 
	\begin{equation}
		\Delta_1 = (1-\gamma)\min_s\left[Q^*(s,\pi^*(s)) - \max_{a\neq \pi^*(s)} Q^*(s,a)\right]/2.\nonumber
	\end{equation}
	If $\Delta< \Delta_1$, there exist $N\in\mathbb{N}$ such that, with probability $1$, for all $t>N$, $\pi_t = \pi^*$. In other words, eventually the RL agent learns the optimal MDP policy $\pi^*$ despite the attacks.
\end{corollary}

\begin{corollary}[Weak Infeasibility Certificate]
	Given attack target policy $\pi^\dagger$, define 
	\begin{equation}
		\Delta_2 = (1-\gamma)\max_s\left[Q^*(s,\pi^*(s)) - \max_{a\in \pi^\dagger(s)} Q^*(s,a)\right]/2.\nonumber
	\end{equation}
	If $\Delta < \Delta_2$, there exist $N\in\mathbb{N}$ such that, with probability $1$, for all $t>N$, $\pi_t(s) \notin \pi^\dagger(s)$ for some $s\in S^\dagger$. In other words, eventually the attacker is unable to enforce $\pi^\dagger$ (though $\pi_t$ may not settle on $\pi^*$ either).
\end{corollary}
Intuitively, an MDP is difficult to attack if its margin $\min_s\left[Q^*(s,\pi^*(s)) - \max_{a\neq \pi^*(s)} Q^*(s,a)\right]$ is large.
This suggests a defense: for RL to be robust against poisoning, the environmental reward signal should be designed such that the optimal actions and suboptimal actions have large performance gaps.

\subsection{Attack Feasibility}
We now show there is a threshold $\Delta_3$ such that
for all $\Delta > \Delta_3$ 
the attacker can enforce $\pi^\dagger$ for all but finite number of rounds.

\begin{theorem}\label{delta3}
 Given a target policy $\pi^\dagger$, define 
 \begin{equation}
\Delta_3 = \frac{1+\gamma}{2}\max_{s\in S^\dagger}[\max_{a\notin \pi^\dagger(s)}Q^*(s,a)- \max_{a\in \pi^\dagger(s)}Q^*(s,a)]_+
\end{equation}
where $[x]_{+}\coloneqq\max(x,0)$. Assume the same conditions on $\alpha_t$ as in~\thmref{delta12}. 
If $\Delta > \Delta_3$, there is a feasible attack policy $\phi^{sas}_{\Delta_3}$. Furthermore, $J_\infty(\phi^{sas}_{\Delta_3})\leq O(L^5)$, where $L$ is the covering number.
\end{theorem}

\begin{algorithm}[ht!]
	\caption{The Non-Adaptive Attack $\phi^{sas}_{\Delta_3}$} \label{alg:delta3}
	\begin{flushleft}
		\textbf{PARAMETERS:} target policy $\pi^\dagger$, agent parameters $\mathcal{A} = (Q_0, \epsilon, \gamma, \{\alpha_t\})$, MDP parameters $\mathcal{M} = (S,A,R,P,\mu_0)$, maximum magnitude of poisoning $\Delta$.\\
		\textbf{def Init($\pi^\dagger, \mathcal{A}, \mathcal{M}$): }
	\end{flushleft}
	\begin{algorithmic}[1]
		\STATE Construct a Q-table $Q^\prime$, where $Q^\prime(s,a)$ is defined as
		$$
		\left\{
		\begin{aligned}
		&Q^*(s,a)+\frac{\Delta}{(1+\gamma)}, &&\mbox{ if } s\in S^\dagger, a\in\pi^\dagger(s)\\
		&Q^*(s,a)-\frac{\Delta}{(1+\gamma)}, &&\mbox{ if } s\in S^\dagger, a\notin \pi^\dagger(s)\\
		&Q^*(s,a), &&\mbox{ if } s\notin S^\dagger
		\end{aligned}
		\right.
		$$
		\STATE Calculate a new reward function 
		$$
		R^\prime(s,a) = Q^\prime(s,a) - \gamma\E{P(s^\prime\mid s, a)}{\max _{a^\prime} Q^\prime(s^\prime, a^\prime)}.
		$$
		\STATE Define the attack policy $\phi^{sas}_{\Delta_3}$ as:
		$$
		\phi^{sas}_{\Delta_3}(s,a)=R^\prime(s,a) - \E{P(s^\prime\mid s, a)}{R(s,a,s)}, \forall s,a.
		$$
	\end{algorithmic}
	\begin{flushleft}
		\textbf{def Attack($\xi_t$): }
	\end{flushleft}
	\begin{algorithmic}[1]
		\STATE Return $\phi^{sas}_{\Delta_3}(s_t,a_t)$
	\end{algorithmic}
\end{algorithm}
Theorem \ref{delta3} is proved by constructing an attack policy $\phi^{sas}_{\Delta_3}(s_t, a_t)$, detailed in Alg. \ref{alg:delta3}.
Note that this attack policy does not depend on $Q_t$.
We call this type of attack \emph{non-adaptive attack}. Under such construction, one can show that Q-learning converges to the target policy $\pi^\dagger$.
Recall the covering number $L$ is the upper bound on the minimum sequence length starting from any $(s,a)$ pair and follow the MDP until all (state, action) pairs appear in the sequence~\cite{even2003learning}.
It is well-known that $\epsilon$-greedy exploration has a covering time $L\leq O(e^{|S|})$~\cite{kearns2002near}. Prior work has constructed examples on which this bound is tight~\cite{jin2018q}. 
We show in appendix \ref{sec:coveringtimeproof} that on our toy example $\epsilon$-greedy indeed has a covering time $O(e^{|S|})$. 
Therefore, the objective value of~\eqref{eq:conc_obj} for non-adaptive attack is upper-bounded by $O(e^{|S|})$.
In other words, the non-adaptive attack is slow.

\subsection{Fast Adaptive Attack (FAA)}
We now show that there is a fast adaptive attack $\phi^{\xi}_{FAA}$ which depends on $Q_t$ and achieves $J_\infty$ polynomial in $|S|$.
The price to pay is a larger attack constraint $\Delta_4$, and the requirement that the attack target states are sparse: $k=|S^\dagger|\leq O(\log|S|)$.
The FAA attack policy $\phi^{\xi}_{FAA}$ is defined in Alg.~\ref{alg:FAA}. 

 Conceptually, the FAA algorithm ranks the target states in descending order by their distance to the starting states, and focusing on attacking one target state at a time. Of central importance is the temporary target policy $\nu_i$, which is designed to navigate the agent to the currently focused target state $s^\dagger_{(i)}$, while not altering the already achieved target actions on target states of earlier rank. This allows FAA to achieve a form of program invariance: after FAA achieves the target policy in a target state $s^\dagger_{(i)}$, the target policy on target state $(i)$ will be preserved indefinitely.
 We provide a more detailed walk-through of Alg.~\ref{alg:FAA} with examples in appendix \ref{sec:FAA}.
\begin{algorithm}[ht!]
	\caption{The Fast Adaptive Attack (FAA)}\label{alg:FAA}
	\begin{flushleft}
		\textbf{PARAMETERS:} target policy $\pi^\dagger$, margin $\eta$, agent parameters $\mathcal{A} = (Q_0, \epsilon, \gamma, \{\alpha_t\})$, MDP parameters $\mathcal{M} = (S,A,R,P,\mu_0)$.\\
		\textbf{def Init($\pi^\dagger, \mathcal{A}, \mathcal{M}, \eta$): }
	\end{flushleft}
	\begin{algorithmic}[1]
		\STATE Given $(s_t, a_t, Q_t)$, define the hypothetical Q-update function without attack as $Q_{t+1}'(s_t,a_t) = (1-\alpha_t) Q_t(s_t,a_t) +\alpha_t \left( r_t + \gamma (1-EOE) \max_{a' \in A} Q_t(s_{t+1}, a') \right)$. 
		\STATE Given $(s_t, a_t, Q_t)$, denote the greedy attack function at $s_t$ w.r.t. a target action set $A_{s_t}$ as $g(A_{s_t})$, defined as
		\begin{eqnarray}
		\left\{
		\begin{array}{ll}
		\frac{1}{\alpha_t}[\max_{a\notin A_{s_t}} Q_t(s_t,a) -\\ 
		\qquad Q'_{t+1}(s_t,a_t)+\eta]_+ & \mbox{ if } a_t \in A_{s_t}\\
		\frac{1}{\alpha_t}[\max_{a\in A_{s_t}} Q_t(s_t,a) -\\ 
		\qquad Q'_{t+1}(s_t,a_t)+\eta]_-
		& \mbox{ if } a_t \notin A_{s_t}.
		\end{array}
		\right.
		\end{eqnarray}
		\STATE Define $\mbox{Clip}_\Delta(\delta) = \min(\max(\delta, -\Delta), \Delta)$.
		\STATE Rank the target states in descending order as $\{s_{(1)}^\dagger, ..., s_{(k)}^\dagger\}$, according to their shortest $\epsilon$-distance to the initial state $\E{s\sim \mu_0}{d^\epsilon(s,s_{(i)})}$.
		\FOR{$i = 1,...,k$}
		\STATE Define the temporary target policy $\nu_i$ as
		\begin{eqnarray}
		\nu_i(s) =
		\left\{
		\begin{array}{ll}
		\pi_{s_{(i)}^\dagger}(s) & \mbox{ if } s\notin \{s_{(j)}^\dagger: j\leq i\}\\
		\pi^\dagger(s)
		& \mbox{ if } s\in \{s_{(j)}^\dagger: j\leq i\}.
		\end{array}
		\right.\nonumber
		\end{eqnarray}
		\ENDFOR
	\end{algorithmic}
	\begin{flushleft}
		\textbf{def Attack($\xi_t$): }
	\end{flushleft}
	\begin{algorithmic}[1]
		\FOR{$i = 1,...,k$}
		\IF{$\argmax_a Q_{t}(s_{(i)}^\dagger,a) \notin \pi^\dagger(s_{(i)}^\dagger)$}
		\STATE Return $\delta_t\leftarrow \mbox{Clip}_{\Delta}(g(\{\nu_i(s_t)\}))$.\\
		\ENDIF
		\ENDFOR
		\STATE Return $\delta_t\leftarrow \mbox{Clip}_{\Delta}(g(\{\pi^\dagger(s_t)\}))$.
	\end{algorithmic}
\end{algorithm}

\begin{definition}
	Define the shortest $\epsilon$-\textbf{distance} from $s$ to $s'$ as
	\begin{align}
	&d_\epsilon(s,s') = \min_{\pi\in \Pi}\E{\pi_{\epsilon}}{T}\\
	&\mbox{s.t. }  s_0 = s, s_T = s', s_t \neq s', \forall t<T\nonumber
	\end{align}
	where $\pi_{\epsilon}$ denotes the epsilon-greedy policy based on $\pi$.
	Since we are in an MDP, there exists a common (partial) policy $\pi_{s'}$ that achieves $d_\epsilon(s,s')$ for all source state $s\in S$. Denote $\pi_{s'}$ as the \textbf{navigation policy} to $s'$.
\end{definition}

\begin{definition}
	The $\epsilon$-\textbf{diameter} of an MDP is defined as the longest shortest $\epsilon$-distance between pairs of states in $S$:
	\begin{equation}
	D_\epsilon = \max_{s,s'\in S} d_\epsilon(s,s')
	\end{equation}
\end{definition}

\begin{theorem} \label{delta4}
	Assume that the learner is running $\epsilon$-greedy Q-learning algorithm on an episodic MDP with $\epsilon$-diameter $D_\epsilon$ and maximum episode length $H$, and the attacker aims at $k$ distinct target states, i.e. $|S^\dagger| = k$. If $\Delta$ is large enough that the $Clip_\Delta()$ function in Alg. \ref{alg:FAA} never takes effect, then $\phi^{\xi}_{FAA}$ is feasible, and we have
	\begin{equation}
	J_\infty(\phi^{\xi}_{FAA}) \leq k\frac{|S||A|H}{1-\epsilon} + \frac{|A|}{1-\epsilon}\left[\frac{|A|}{\epsilon}\right]^kD_\epsilon,
	\end{equation}
\end{theorem}

How large is $D_\epsilon$? For MDPs with underlying structure as undirected graphs, such as the grid worlds, it is shown that the expected hitting time of a uniform random walk is bounded by $O(|S|^2)$\cite{lawler1986expected}. Note that the random hitting time tightly upper bounds the optimal hitting time, a.k.a. the $\epsilon$-diameter $D_\epsilon$, and they match when $\epsilon = 1$. This immediately gives us the following result:

\begin{corollary} If in addition to the assumptions of Theorem \ref{delta4}, the maximal episode length $H = O(|S|)$, then $J_\infty(\phi^{\xi}_{FAA})\leq O(e^k|S|^2|A|)$ in Grid World environments. When the number of target states is small, i.e. $k \leq O(\log |S|)$, $J_\infty(\phi^{\xi}_{FAA})\leq O(\mbox{\emph{poly}}(|S|))$. \label{cor:gw}
\end{corollary}

\textbf{Remark 2:} Theorem \ref{delta4} and Corollary \ref{cor:gw} can be thought of as defining an implicit $\Delta_4$, such that for any $\Delta>\Delta_4$, the clip function in Alg. \ref{alg:FAA} never take effect, and $\phi^{\xi}_{FAA}$ achieves polynomial cost.

\subsection{Illustrating Attack (In)feasibility $\Delta$ Thresholds}
The theoretical results developed so far can be summarized as a diagram in Figure \ref{fig:numberline}.
We use the chain MDP in Figure~\ref{fig:running_example} to illustrate the four thresholds $\Delta_1, \Delta_2, \Delta_3, \Delta_4$ developed in this section.
On this MDP and with this attack target policy $\pi^\dagger$, we found that $\Delta_1 = \Delta_2 = 0.0069$.
The two matches because this $\pi^\dagger$ is the easiest to achieve in terms of having the smallest upperbound $\Delta_2$.
Attackers whose poison magnitude $|\delta_t| < \Delta_2$ will not be able to enforce the target policy $\pi^\dagger$ in the long run.

We found that $\Delta_3 = 0.132$.
We know that $\phi^{sas}_{\Delta_3}$ should be feasible if $\Delta>\Delta_3$.
To illustrate this, we ran $\phi^{sas}_{\Delta_3}$ with $\Delta = 0.2 > \Delta_3$ for 1000 trials and obtained estimated $J_{10^5}(\phi^{sas}_{\Delta_3}) = 9430$.
The fact that $J_{10^5}(\phi^{sas}_{\Delta_3}) \ll T=10^5$ is empirical evidence that $\phi^{sas}_{\Delta_3}$ is feasible.
We found that $\Delta_4 = 1$ by simulation. 
The adaptive attack $\phi^{\xi}_{FAA}$ constructed in Theorem \ref{delta4} should be feasible with $\Delta = \Delta_4 = 1$.
We run $\phi^{\xi}_{FAA}$  for 1000 trials and observed $J_{10^5}(\phi^{\xi}_{FAA}) = 30.4 \ll T$, again verifying the theorem. Also observe that $J_{10^5}(\phi^{\xi}_{FAA})$ is much smaller than $J_{10^5}(\phi^{sas}_{\Delta_3})$, verifying the foundamental difference in attack efficiency between the two attack policies as shown in Theorem \ref{delta3} and Corollary \ref{cor:gw}.

While FAA is able to force the target policy in polynomial time, it's not necessarily the optimal attack strategy. Next, we demonstrate how to solve for the optimal attack problem in practice, and empirically show that with the techniques from Deep Reinforcement Learning (DRL), we can find efficient attack policies in a variety of environments.

\section{Attack RL with RL}
\begin{figure}[t!]
	\centering
	\includegraphics[width=1.0\columnwidth]{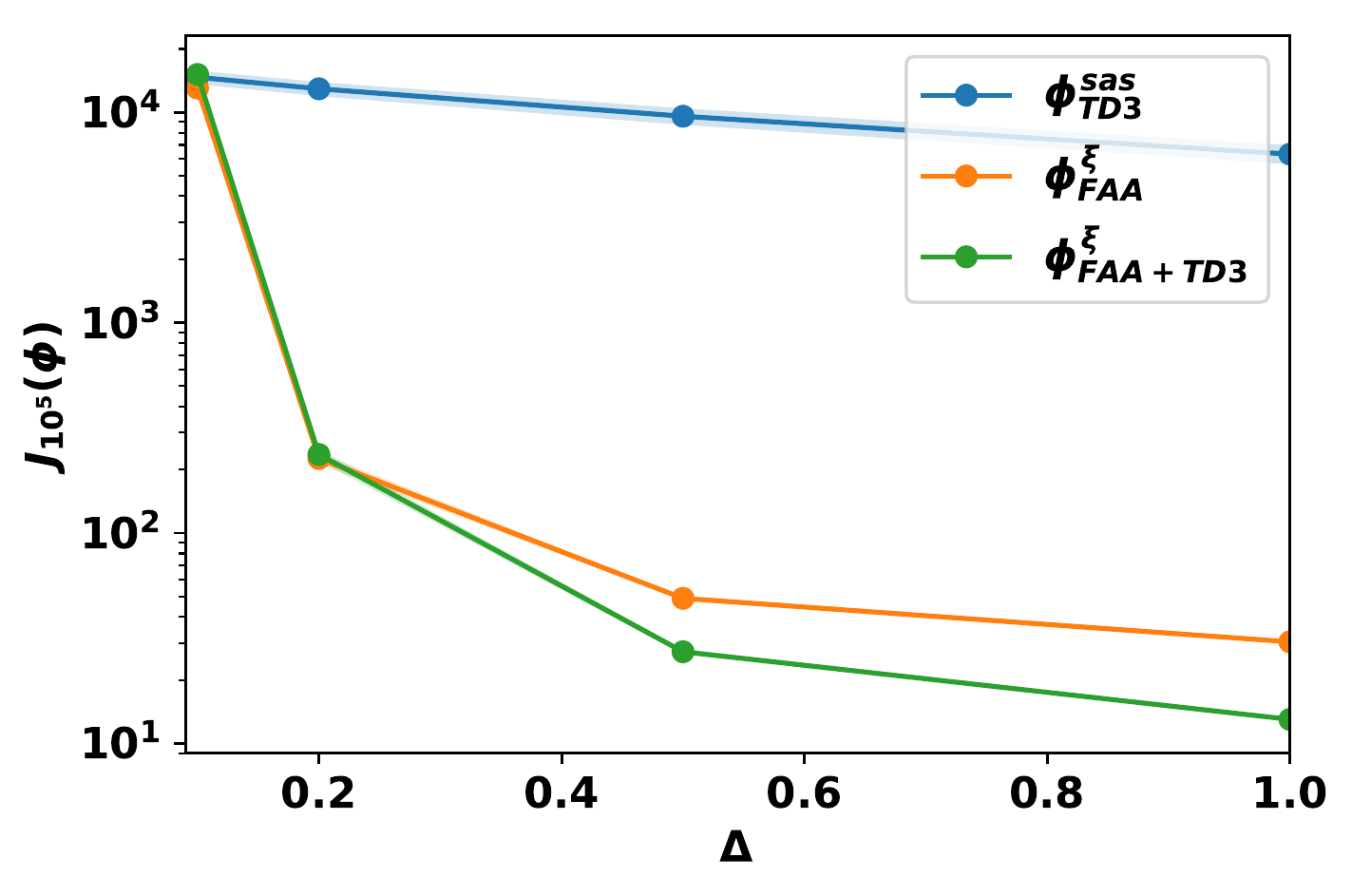}
	\caption{Attack cost $J_{10^5}(\phi)$ on different $\Delta$'s. Each curve shows mean $\pm 1$ standard error over 1000 independent test runs.}
	\label{fig:exp_deltas}
\end{figure}
The attack policies $\phi^{sas}_{\Delta_3}$ and $\phi^{\xi}_{FAA}$ were manually constructed for theoretical analysis.
Empirically, though, they do not have to be the most effective attacks under the relevant $\Delta$ constraint.

In this section, we present our key computational insight: the attacker can find an effective attack policy by relaxing the attack problem~\eqref{eq:conc_obj} so that the relaxed problem can be effectively solved with RL.
Concretely, consider the higher-level attack MDP $\mathcal N=(\Xi, \Delta, \rho, \tau)$ and the associated optimal control problem:
\begin{itemize}[leftmargin=*]
	\item 
	The attacker observes the attack state $\xi_t \in \Xi$.
	\item
	The attack action space is $\{\delta_t \in \R: |\delta_t| \leq \Delta\}$.
	\item 
	The original attack loss function $\mathbf{1}{[Q_t\notin \mathcal Q^\dagger]}$ is a 0-1 loss that is hard to optimize.
	We replace it with a continuous surrogate loss function $\rho$ that measures how close the current agent Q-table $Q_t$ is to the target Q-table set:
	\begin{equation}
	\label{eq:surrogate}
	\rho(\xi_t) =
	\sum_{s\in S^\dagger}\left[
	\max_{a\notin \pi^\dagger(s)} Q_t(s, a) 
	- \max_{a\in \pi^\dagger(s)} Q_t(s, a)  + \eta
	\right]_{+}
	\end{equation}
	where $\eta>0$ is a margin parameter to encourage that $\pi^\dagger(s)$ is strictly preferred over $A\backslash \pi^\dagger(s)$ with no ties.
	\item
	The attack state transition probability is defined by $\tau(\xi_{t+1} \mid \xi_t, \delta_t)$.
	Specifically, the new attack state $\xi_{t+1}=(s_{t+1}, a_{t+1}, s_{t+2}, r_{t+1}, Q_{t+1})$ is generated as follows:
	\begin{itemize}[leftmargin=*]
		\item $s_{t+1}$ is copied from $\xi_t$ if not the end of episode, else $s_{t+1}\sim \mu_0$.
		\item $a_{t+1}$ is the RL agent's exploration action drawn according to~\eqref{eq:explorationpolicy}, note it involves $Q_{t+1}$.
		\item $s_{t+2}$ is the RL agent's new state drawn according to the MDP transition probability $P(\cdot \mid s_{t+1}, a_{t+1})$.
		\item $r_{t+1}$ is the new (not yet poison) reward according to MDP $R(s_{t+1}, a_{t+1}, s_{t+2})$.
		\item The attack $\delta_t$ happens.  The RL agent updates $Q_{t+1}$ according to~\eqref{eq:Qlearning}.
	\end{itemize}
\end{itemize}
With the higher-level attack MDP $\mathcal N$, we relax the optimal attack problem~\eqref{eq:conc_obj} into
\begin{equation}
\phi^* = \argmin_{\phi} \mathbb E_{\phi} \sum_{t=0}^\infty\rho(\xi_t)
\label{eq:objective}
\end{equation}
One can now solve~\eqref{eq:objective} using Deep RL algorithms.  In this paper, we choose Twin Delayed DDPG (TD3)~\cite{fujimoto2018addressing}, a state-of-the-art algorithm for continuous action space.
We use the same set of hyperparameters for TD3 across all experiments, described in appendix \ref{sec:TD3params}.

\section{Experiments}
\begin{figure}[t]
	\centering
	\includegraphics[width=1.0 \columnwidth]{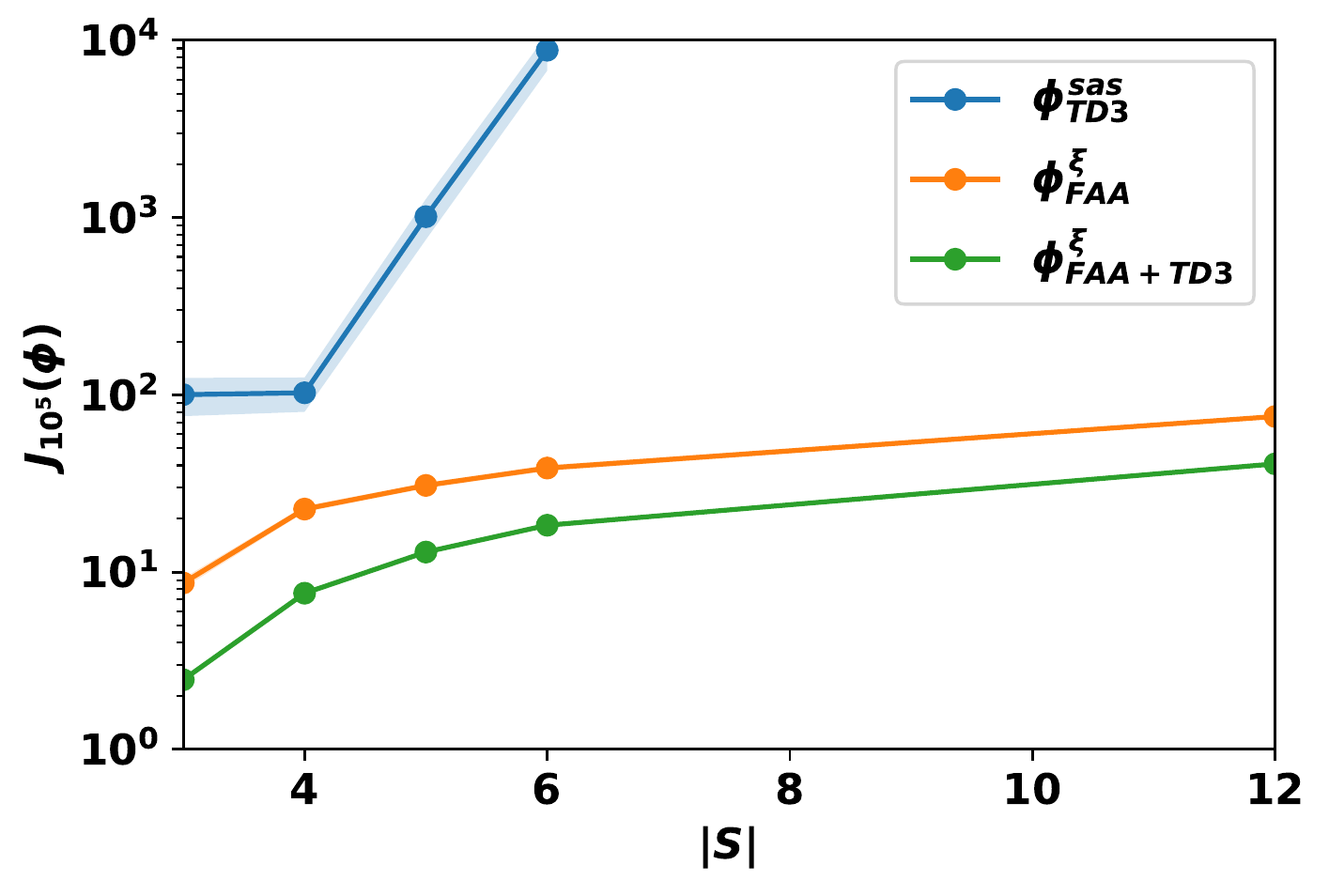}
	\caption{Attack performances on the chain MDPs of different lengths.  Each curve shows mean $\pm 1$ standard error over 1000 independent test runs.}
	\label{fig:exp_length}
\end{figure}
In this section, We make empirical comparisons between a number of attack policies $\phi$:
We use the naming convention where the superscript denotes non-adaptive or adaptive policy: $\phi^{sas}$ depends on $(s_t,a_t, s_{t+1})$ but not $Q_t$. Such policies have been extensively used in the reward shaping literature and prior work~\cite{ma2019policy, huang2019deceptive} on reward poisoning; $\phi^{\xi}$ depends on the whole attack state $\xi_t$. We use the subscript to denote how the policy is constructed.
Therefore, $\phi^{\xi}_{TD3}$ is the attack policy found by solving~\eqref{eq:objective} with TD3;
$\phi^{\xi}_{FAA+TD3}$ is the attack policy found by TD3 initialized from FAA (Algorithm~\ref{alg:FAA}), where TD3 learns to provide an additional $\delta_t'$ on top of the $\delta_t$ generated by $\phi^{\xi}_{FAA}$, and the agent receives $r_t+\delta_t+\delta_t'$ as reward;
$\phi^{sas}_{TD3}$ is the attack policy found using TD3 with the restriction that the attack policy only takes $(s_t,a_t,s_{t+1})$ as input.

In all of our experiments, we assume a standard Q-learning RL agent with parameters: $Q_0 = 0^{S\times A}$, $\epsilon = 0.1, \gamma=0.9, \alpha_t = 0.9, \forall t$. 
The plots show $\pm 1$ standard error around each curve (some are difficult to see).
We will often evaluate an attack policy $\phi$ using a Monte Carlo estimate of the 0-1 attack cost $J_T(\phi)$ for $T=10^5$, which approximates the objective $J_\infty(\phi)$ in~\eqref{eq:conc_obj}. 

\subsection{Efficiency of Attacks across different $\Delta$'s}

Recall that $\Delta>\Delta_3$, $\Delta>\Delta_4$ are sufficient conditions for manually-designed attack policies $\phi^{sas}_{\Delta_3}$ and $\phi^{\xi}_{FAA}$ to be feasible, but they are not necessary conditions. In this experiment, we empirically investigate the feasibilities and efficiency of non-adaptive and adaptive attacks across different $\Delta$ values.

We perform the experiments on the chain MDP in Figure~\ref{fig:running_example}. Recall that on this example, $\Delta_3 = 0.132$ and $\Delta_4 = 1$ (implicit). We evaluate across 4 different $\Delta$ values, $[0.1, 0.2, 0.5, 1]$, covering the range from $\Delta_3$ to $\Delta_4$. The result is shown in Figure \ref{fig:exp_deltas}.

We are able to make several interesting observations:\\
(1) All attacks are feasible ($y$-axis $\ll T$), even when $\Delta$ falls under the thresholds $\Delta_3$ and $\Delta_4$ for corresponding methods. This suggests that the feasibility thresholds are not tight.\\
(2) For non-adaptive attacks, as $\Delta$ increases the best-found attack policies $\phi^{sas}_{TD3}$ achieve small improvement, but generally incur a large attack cost.\\
(3) Adaptive attacks are very efficient when $\Delta$ is large. At $\Delta = 1$, the best adaptive attack $\phi^\xi_{FAA+TD3}$ achieves a cost of merely 13 (takes 13 steps to always force $\pi^\dagger$ on the RL agent). 
However, as $\Delta$ decreases the performance quickly degrades. 
At $\Delta = 0.1$ adaptive attacks are only as good as non-adaptive attacks. 
This shows an interesting transition region in $\Delta$ that our theoretical analysis does not cover.

\subsection{Adaptive Attacks are Faster}
\begin{figure}[t]
	\centering
	\includegraphics[width=.7\columnwidth]{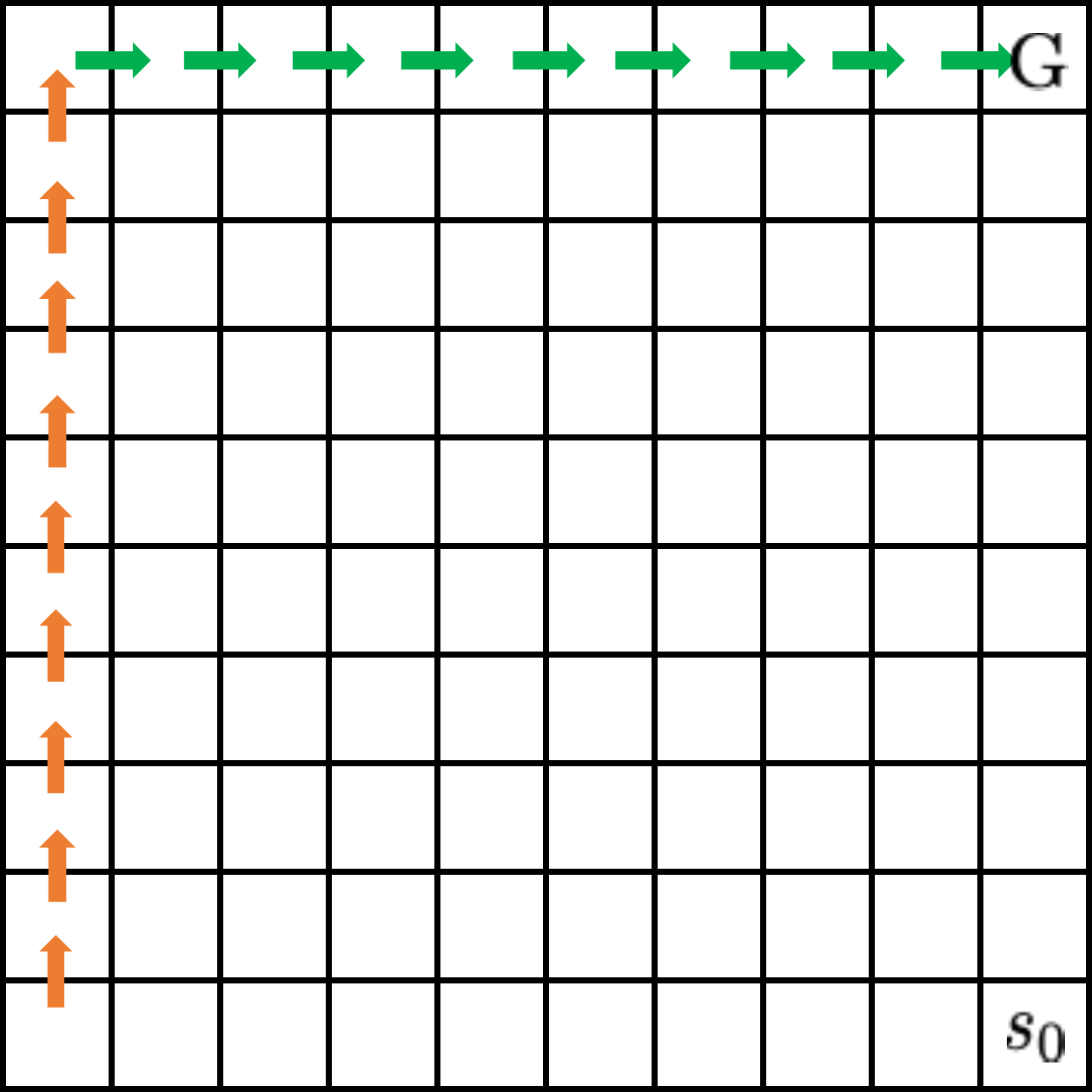}
	\caption{The $10\times 10$ Grid World. $s_0$ is the starting state and $G$ the terminal goal.
		Each move has a $-0.1$ negative reward, and a $+1$ reward for arriving at the goal.  
		We consider two partial target policies: $\pi^\dagger_1$ marked by the green arrows, and $\pi^\dagger_2$ by \emph{both} the green and the orange arrows.}
	\label{fig:10x10}
\end{figure}
In this experiment, we empirically verify that, while both are feasible, adaptive attacks indeed have an attack cost $O(\mbox{Poly}|S|)$ while non-adaptive attacks have $O(e^{|S|})$.
The 0-1 costs $1[\pi_t \neq \pi^\dagger]$ are in general incurred at the beginning of each $t=0\ldots T$ run.
In other words, adaptive attacks achieve $\pi^\dagger$ faster than non-adaptive attacks.
We use several chain MDPs similar to Figure~\ref{fig:running_example} but with increasing number of states $|S| = 3,4,5,6,12$.
We provide a large enough $\Delta = 2 \gg \Delta_4$ to ensure the feasibility of all attack policies. 
\begin{figure*}[ht!]
	\begin{subfigure}{0.24\textwidth}
		\centering
		\includegraphics[width=1\columnwidth]{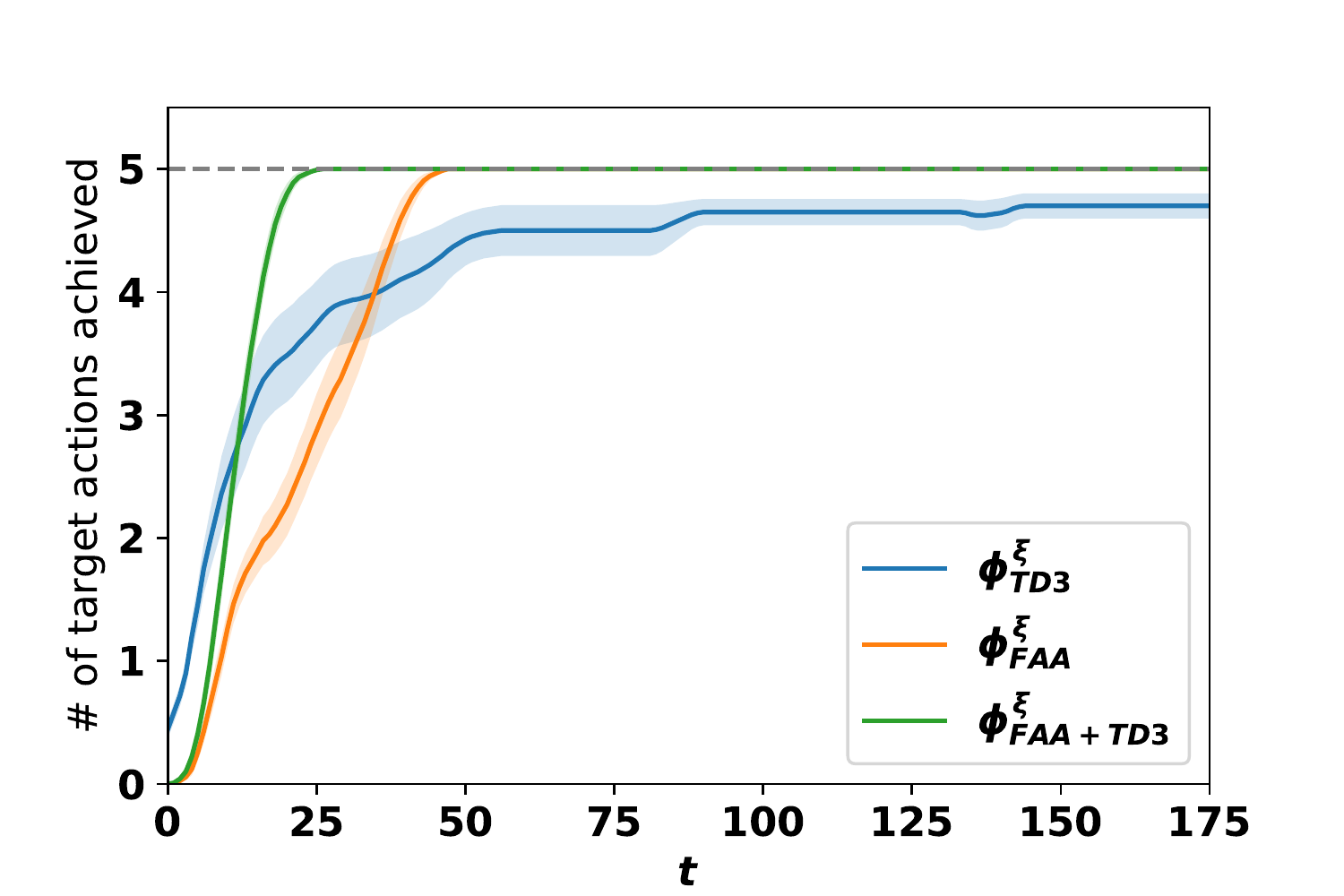}
		\caption{6-state chain MDP}
		\label{fig:ablation1}
	\end{subfigure}
	\begin{subfigure}{0.24\textwidth}
		\centering
		\includegraphics[width=1\columnwidth]{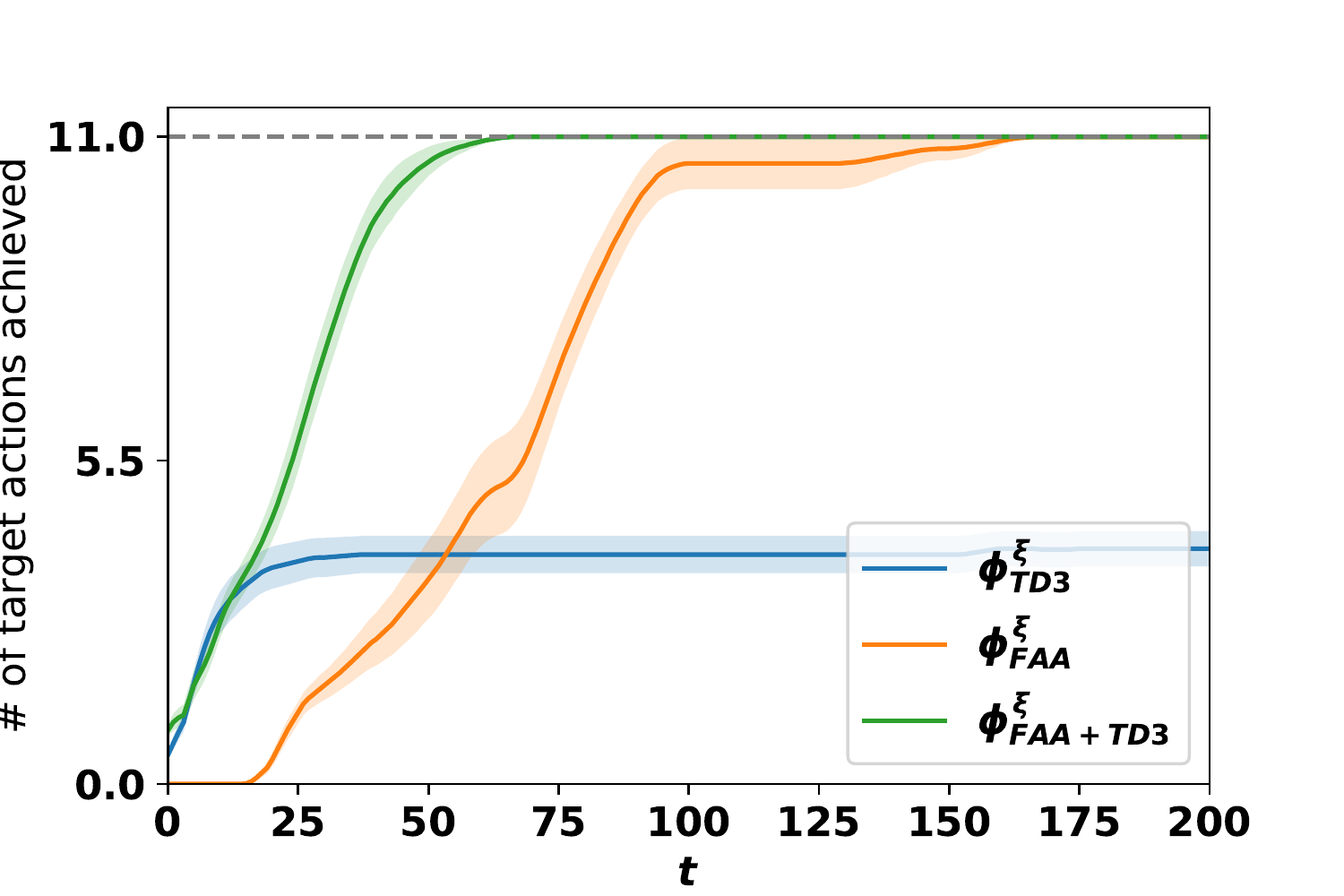}
		\caption{12-state chain MDP}
		\label{fig:ablation2}
	\end{subfigure}
	\begin{subfigure}{0.24\textwidth}
		\centering
		\includegraphics[width=1\columnwidth]{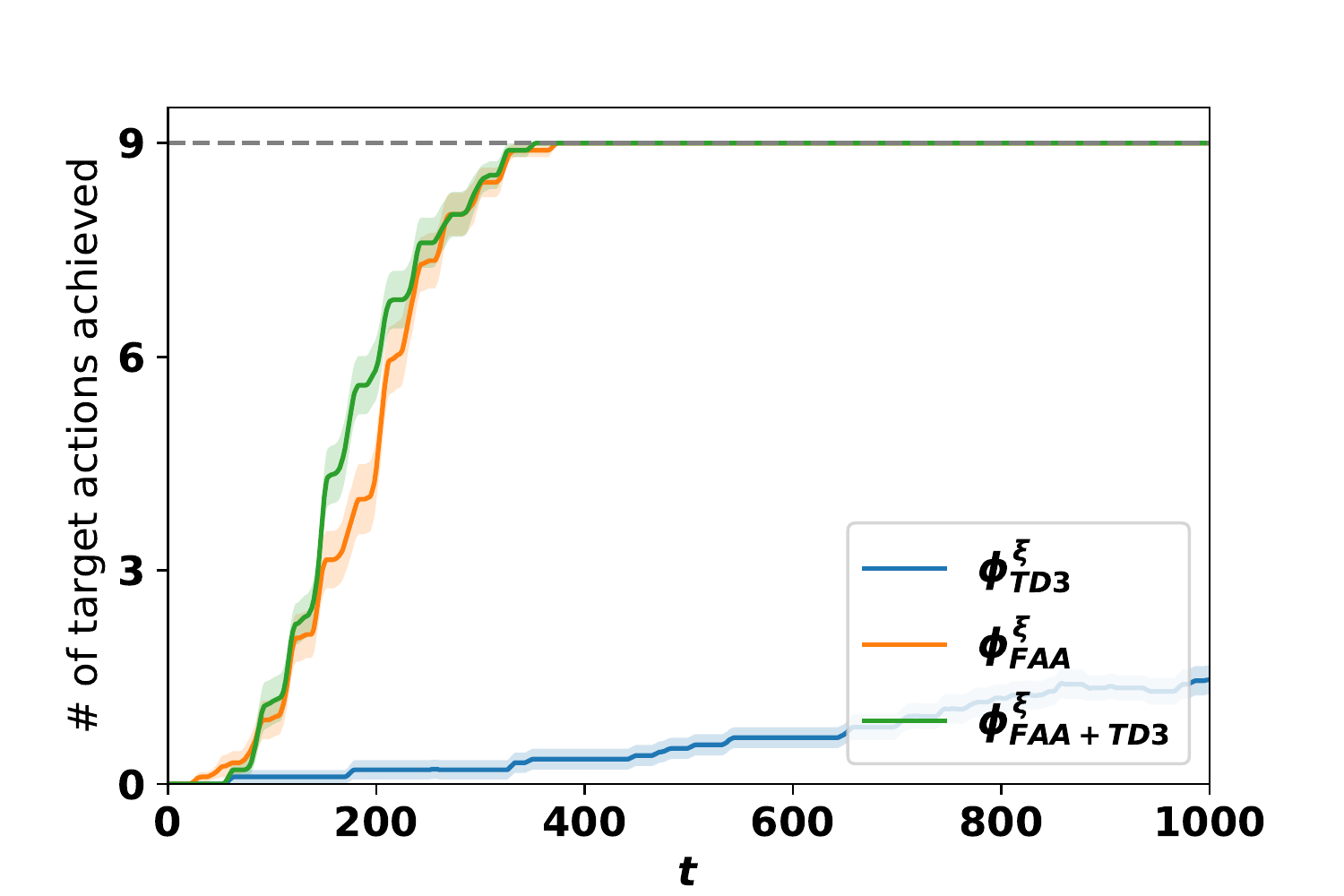}
		\caption{$10\times 10$ MDP with $\pi^\dagger_{1}$.}
		\label{fig:ablation3}
	\end{subfigure}
	\begin{subfigure}{0.24\textwidth}
		\centering
		\includegraphics[width=1\columnwidth]{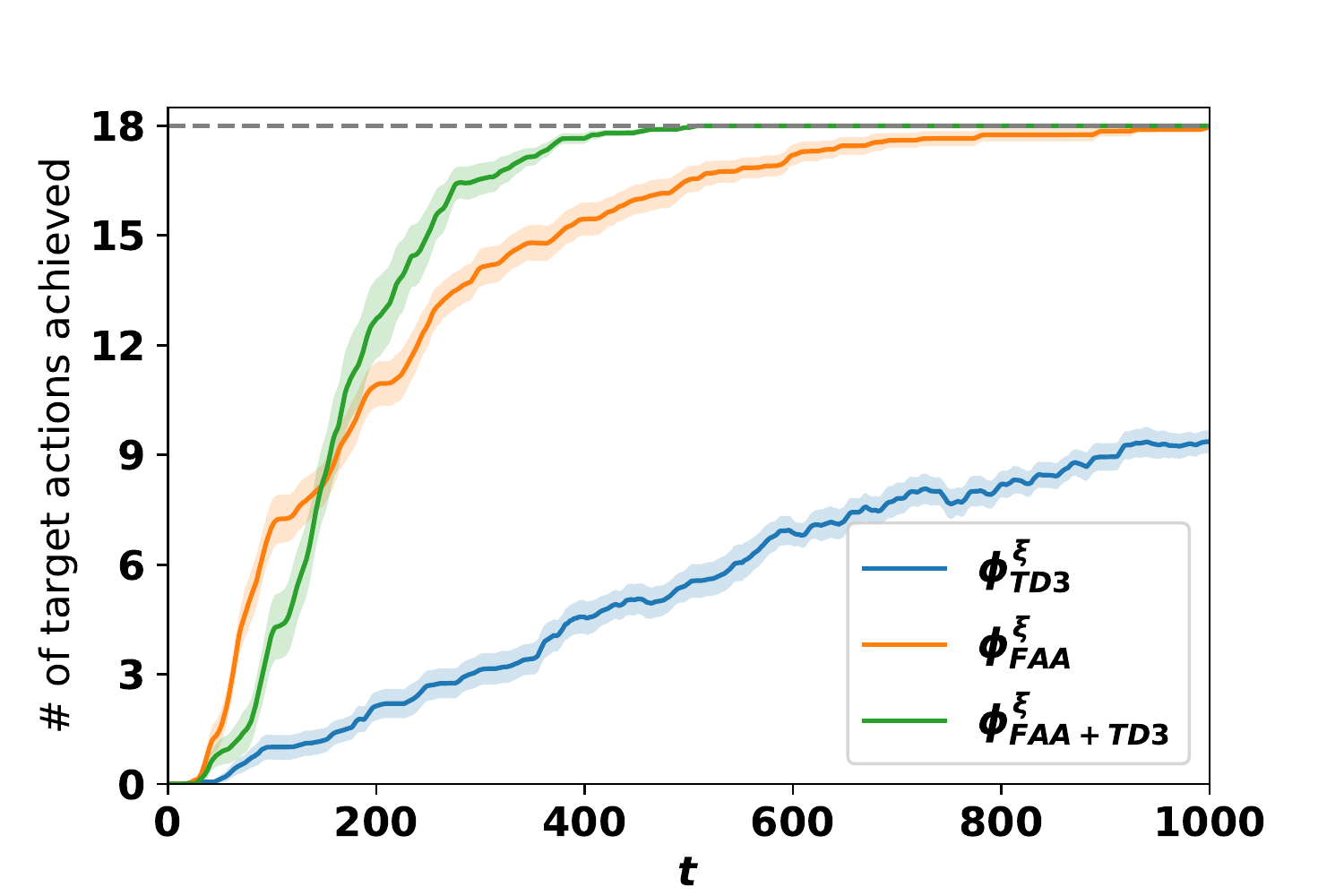}
		\caption{$10\times 10$ MDP with $\pi^\dagger_{2}$.}
		\label{fig:ablation4}
	\end{subfigure}
	\caption{Experiment results for the ablation study. Each curve shows mean $\pm 1$ standard error over 20 independent test runs. The gray dashed lines indicate the total number of target actions.}
	\label{fig:ablation}
\end{figure*}
The result is shown in Figure~\ref{fig:exp_length}.
The best-found non-adaptive attack $\phi^{sas}_{TD3}$ is approximately straight on the log-scale plot, suggesting attack cost $J$ growing exponentially with MDP size $|S|$. 
In contrast, the two adaptive attack polices $\phi^\xi_{FAA}$ and $\phi^\xi_{FAA+TD3}$ actually achieves attack cost linear in $|S|$.
This is not easy to see from this log-scaled plot; We reproduce Figure~\ref{fig:exp_length} without the log scale in the appendix \ref{sec:normal scale}, where the linear rate can be clearly verified. This suggests that the upperbound developed in Theorem \ref{delta4} and Corollary \ref{cor:gw} can be potentially improved.

\subsection{Ablation Study}
In this experiment, we compare three adaptive attack policies: $\phi^\xi_{TD3}$ the policy found by out-of-the-box TD3, $\phi^\xi_{FAA}$ the manually designed FAA policy, and $\phi^\xi_{FAA+TD3}$ the policy found by using FAA as initialization for TD3. 

We use three MDPs: a 6-state chain MDP, a 12-state chain MDP, and a $10\times 10$ grid world MDP..
The $10\times 10$ MDP has two separate target policies $\pi^\dagger_1$ and $\pi^\dagger_2$, see Figure~\ref{fig:10x10}.

For evaluation, we compute the number of target actions achieved $|\{s \in S^\dagger: \pi_t(s)\in\pi^\dagger(s)\}|$ as a function of $t$. This allows us to look more closely into the progress made by an attack.
The results are shown in Figure \ref{fig:ablation}.

First, observe that across all 4 experiments, attack policy $\phi^{\xi}_{TD3}$ found by out-of-the-box TD3 never succeeded in achieving all target actions.
This indicates that TD3 alone cannot produce an effective attack. We hypothesize that this is due to a lack of effective exploration scheme:
when the target states are sparse ($|S^\dagger| \ll |S|)$ it can be hard for TD3 equiped with Gaussian exploration noise to locate all target states. As a result, the attack policy found by vanilla TD3 is only able to achieve the target actions on a subset of frequently visited target states.

Hand-crafted $\phi^\xi_{FAA}$ is effective in achieving the target policies, as is guaranteed by our theory. Nevertheless, we found that $\phi^\xi_{FAA+TD3}$ always improves upon $\phi^\xi_{TD3}$.
Recall that we use FAA as the initialization and then run TD3.
This indicates that TD3 can be highly effective with a good initialization, which effectively serves as the initial exploration policy that allows TD3 to locate all the target states.

Of special interest are the two experiments on the $10\times 10$ Grid World with different target policies.
Conceptually, the advantage of the adaptive attack is that the attacker can perform explicit navigation to lure the agent into the target states. An efficient navigation policy that leads the agent to all target states will make the attack very efficient.
Observe that in Figure \ref{fig:10x10}, both target polices form a chain, so that if the agent starts at \emph{the beginning of the chain}, the target actions naturally lead the agent to the subsequent target states, achieving efficient navigation.

Recall that the FAA algorithm prioritizes the target states farthest to the starting state. In the $10\times 10$ Grid World, the farthest state is the top-left grid. For target states $S^\dagger_1$, the top-left grid turns out to be the beginning of the \emph{target chain}. As a result, $\phi^\xi_{FAA}$ is already very efficient, and $\phi^\xi_{FAA+TD3}$ couldn't achieve much improvement, as shown in \ref{fig:ablation3}. On the other hand, for target states $S^\dagger_2$, the top-left grid is in the middle of the target chain, which makes $\phi^\xi_{FAA}$ not as efficient. In this case, $\phi^\xi_{FAA+TD3}$ makes a significant improvement, successfully forcing the target policy in about 500 steps, whereas it takes $\phi^\xi_{FAA}$ as many as 1000 steps, about twice as long as $\phi^\xi_{FAA+TD3}$.

\section{Conclusion}
In this paper, we studied the problem of reward-poisoning attacks against reinforcement-learning agents. Theoretically, we provide robustness certificates that guarantee the truthfulness of the learned policy when the attacker's constraint is stringent. When the constraint is loose, we show that by being adaptive to the agent's internal state, the attacker can force the target policy in polynomial time, whereas a naive non-adaptive attack takes exponential time.
Empirically, we formulate that the reward poisoning problem as an optimal control problem on a higher-level attack MDP, and developed computational tools based on DRL that is able to find efficient attack policies across a variety of environments.
\section*{Acknowledgments}
This work is supported in part by NSF 1545481, 1623605, 1704117, 1836978 and the MADLab AF Center of Excellence FA9550-18-1-0166.
\newpage
\bibliography{online_attack_on_RL}
\bibliographystyle{icml2020}

\clearpage
\newpage
\onecolumn
\appendix
\appendixpage
\section{Proof of Theorem \ref{delta12}}
\label{sec:delta12proof}
\begin{proof}
	Consider two MDPs with reward functions defined as $R+\Delta$ and $R-\Delta$, denote the Q table corresponding to them as $Q_{+\Delta}$ and $Q_{-\Delta}$, respectively. Let $\{(s_t,a_t)\}$ be any instantiated trajectory of the learner corresponding to the attack policy $\phi$. By assumption, $\{(s_t,a_t)\}$ visits all $(s,a)$ pairs infinitely often and $\alpha_t$'s satisfy $\sum \alpha_t=\infty$ and $\sum \alpha_t^2<\infty$. Assuming now that we apply Q-learning on this particular trajectory with reward given by $r_t + \Delta$, standard Q-learning convergence applies and we have that $Q_{t,+ \Delta}\rightarrow Q_{+\Delta}$ and similarly, $Q_{t,- \Delta}\rightarrow Q_{-\Delta}$ \cite{melo2001convergence}.
	
	Next, we want to show that $Q_t(s,a)\leq Q_{t,+ \Delta}(s,a)$ for all $s\in S, a\in A$ and for all $t$. We prove by induction. First, we know $Q_0(s,a) = Q_{0,+ \Delta}(s,a)$. Now, assume that $Q_k(s,a)\leq Q_{k,+ \Delta}(s,a)$. We have
	\begin{eqnarray}
	&&Q_{k+1,+ \Delta}(s_{k+1},a_{k+1})\\
	&=& (1-\alpha_{k+1})Q_{k,+ \Delta}(s_{k+1},a_{k+1}) + \alpha_{k+1}\left(r_{k+1} + \Delta + \gamma \max_{a' \in A} Q_{k,+ \Delta}(s'_{k+1}, a')\right)\\
	&\geq&(1-\alpha_{k+1})Q_{k}(s_{k+1},a_{k+1}) +\alpha_{k+1}\left(r_{k+1} + \delta_{k+1} + \gamma \max_{a' \in A} Q_k(s'_{k+1}, a')\right)\\
	&=& Q_{k+1}(s_{k+1},a_{k+1}),
	\end{eqnarray}
	which established the induction. Similarly, we have $Q_t(s,a)\geq Q_{t,- \Delta}(s,a)$. Since $Q_{t,+ \Delta}\rightarrow Q_{+\Delta}$, $Q_{t,-\Delta}\rightarrow Q_{-\Delta}$, we have that for large enough $t$,
	\begin{eqnarray}
	Q_{-\Delta}(s,a)\leq Q_t(s,a) \leq Q_{+\Delta}, \forall s\in S,a\in A.
	\end{eqnarray}
	Finally, it's not hard to see that $Q_{+\Delta}(s,a) = Q^*(s,a) + \frac{\Delta}{1-\gamma}$ and $Q_{-\Delta}(s,a) = Q^*(s,a) - \frac{\Delta}{1-\gamma}$. This concludes the proof.
\end{proof}

\section{Proof of Theorem \ref{delta3}}
\label{sec:delta3proof}
\begin{proof}
	We provide a constructive proof. We first design an attack policy $\phi$, and then show that $\phi$ is a \textit{strong attack}.
	For the purpose of finding a strong attack, it suffices to restrict the constructed $\phi$ to depend only on $(s,a)$ pairs, which is a special case of our general attack setting. Specifically, for any $\Delta>\Delta_3$, we define the following $Q^\prime$:
	\begin{equation}
	Q^\prime(s,a)=\left\{
	\begin{aligned}
	&Q^*(s,a)+\frac{\Delta}{(1+\gamma)}, &&\forall s\in S^\dagger, a\in\pi^\dagger(s),\\
	&Q^*(s,a)-\frac{\Delta}{(1+\gamma)}, &&\forall s\in S^\dagger, a\notin \pi^\dagger(s),\\
	& Q^*(s,a), \forall s\notin S^\dagger, a,
	\end{aligned}
	\right.
	\end{equation}
	where $Q^*(s,a)$ is the original optimal value function without attack.
	We will show $Q^\prime\in\Q^\dagger$, i.e., the constructed $Q^\prime$ induces the target policy. For any $s\in S^\dagger$, let $a^\dagger\in\argmax_{a\in \pi^\dagger(s)}Q^*(s,a)$, a best target action desired by the attacker under the original value function $Q^*$. We next show that $a^\dagger$ becomes the optimal action under $Q^\prime$. Specifically, $\forall a^\prime\notin \pi^\dagger(s)$, we have
	\begin{eqnarray}
	Q^\prime(s,a^\dagger)&=&Q^*(s,a^\dagger)+\frac{\Delta}{(1+\gamma)}\\
	&=&Q^*(s,a^\dagger)-Q^*(s,a^\prime)+\frac{2\Delta}{(1+\gamma)}+Q^*(s,a^\prime)-\frac{\Delta}{(1+\gamma)}\\
	&=&Q^*(s,a^\dagger)-Q^*(s,a^\prime)+\frac{2\Delta}{(1+\gamma)}+Q^\prime(s,a^\prime),
	\end{eqnarray}
	Next note that
	\begin{eqnarray}
	\Delta>\Delta_3&\ge& \frac{1+\gamma}{2}[\max_{a\notin \pi^\dagger(s)}Q^*(s,a)- \max_{a\in \pi^\dagger(s)}Q^*(s,a)]\\
	&=&\frac{1+\gamma}{2}[\max_{a\notin \pi^\dagger(s)}Q^*(s,a)- Q^*(s,a^\dagger)]\\
	&\ge& \frac{1+\gamma}{2}[Q^*(s,a^\prime)- Q^*(s,a^\dagger)],
	\end{eqnarray}
	which is equivalent to
	\begin{equation}
	Q^*(s,a^\dagger) -Q^*(s,a^\prime)> - \frac{2\Delta}{1+\gamma},
	\end{equation}
	thus we have

	\begin{eqnarray}
	Q^\prime(s,a^\dagger)&=&Q^*(s,a^\dagger)-Q^*(s,a^\prime)+\frac{2\Delta}{(1+\gamma)}+Q^\prime(s,a^\prime)\\
	&>& 0+Q^\prime(s,a^\prime)=Q^\prime(s,a^\prime).
	\end{eqnarray}
	This shows that under $Q^\prime$, the original best target action $a^\dagger$ becomes better than all non-target actions, thus $a^\dagger$ is optimal and $Q^\prime\in\Q^\dagger$. According to Proposition 4 in~\cite{ma2019policy}, the Bellman optimality equation induces a unique reward function $R^\prime(s,a)$ corresponding to $Q^\prime$:
	\begin{equation}
	R^\prime(s,a) = Q^\prime(s,a) - \gamma\sum_{s^\prime}P(s^\prime\mid s, a)\max _{a^\prime} Q^\prime(s^\prime, a^\prime).
	\end{equation}
	We then construct our attack policy $\phi^{sas}_{\Delta_3}$ as:
	\begin{equation}
	\phi^{sas}_{\Delta_3}(s,a)=R^\prime(s,a) - R(s,a), \forall s,a.
	\end{equation}
	The $\phi^{sas}_{\Delta_3}(s,a)$ results in that the reward function after attack appears to be $R^\prime(s,a)$ from the learner's perspective. This in turn guarantees that the learner will eventually learn $Q^\prime$, which achieves the target policy. Next we show that under $\phi^{sas}_{\Delta_3}(s,a)$, the objective value~\eqref{eq:conc_obj} is finite, thus the attack is feasible. To prove feasibility, we consider adapting Theorem 4 in~\cite{even2003learning}, re-stated as below.
	\begin{lemma}[Even-Dar \& Mansour]
	Assume the attack is $\phi^{sas}_{\Delta_3}(s,a)$ and let $Q_t$ be the value of the Q-learning algorithm using polynomial learning rate $\alpha_t = (\frac{1}{1+t})^\omega$ where $\omega\in(\frac{1}{2},1]$. Then with probability at least $1-\delta$, we have $\|Q_T-Q^\prime\|_\infty\le \tau$ with
	\begin{equation}
	T=\Omega\left(L^{3+\frac{1}{\omega}}\frac{1}{\tau^2}(\ln\frac{1}{\delta\tau})^{\frac{1}{\omega}}+L^{\frac{1}{1-\omega}}\ln\frac{1}{\tau}\right),
	\end{equation}
	\end{lemma}
	Note that $Q^\dagger$ is an open set and $Q^\prime\in\Q^\dagger$. This implies that one can pick a small enough $\tau_0>0$ such that $\|Q_T-Q^\prime\|_\infty\le\tau_0$ implies $Q_T\in Q^\dagger$.	 From now on we fix this $\tau_0$, thus the bound in the above theorem becomes 
	\begin{equation}
	T=\Omega\left(L^{3+\frac{1}{\omega}}(\ln\frac{1}{\delta})^{\frac{1}{\omega}}+L^{\frac{1}{1-\omega}}\right).
	\end{equation}
	As the authors pointed out in~\cite{even2003learning}, the $\omega$ that leads to the tightest lower bound on $T$ is around 0.77. Here for our purpose of proving feasibility, it is simpler to let $\omega\approx \frac{1}{2}$ to obtain a loose lower bound on $T$ as below
	\begin{equation}
	T = \Omega\left(L^5(\ln\frac{1}{\delta})^2\right).
	\end{equation}
	Now we represent $\delta$ as a function of $T$ to obtain that $\forall T>0$,
	\begin{equation}
	P[\|Q_T-Q^\prime\|_\infty>\tau_0]\le C\exp(-L^{-\frac{5}{2}}T^{\frac{1}{2}}).
	\end{equation}
	Let $e_t=\ind{\|Q_t-Q^\prime\|_\infty>\tau_0}$, then we have
	\begin{eqnarray}
	\E{\phi^{sas}_{\Delta_3}}{\sum_{t=1}^{\infty} \mathbf{1}[{Q_t \notin \mathcal Q^\dagger}]}&\le & \E{\phi^{sas}_{\Delta_3}}{\sum_{t=1}^{\infty} e_t}\\
	&=&\sum_{t=1}^\infty P[\|Q_T-Q^\prime\|_\infty>\tau_0]\le\sum_{t=1}^\infty C\exp(-L^{-\frac{5}{2}}t^{\frac{1}{2}})\\
	&\le& \int_{t=0}^\infty C\exp(-L^{-\frac{5}{2}}t^{\frac{1}{2}})dt=2C L^5,
	\end{eqnarray}
	which is finite. Therefore the attack is feasible.

	It remains to validate that $\phi^{sas}_{\Delta_3}$ is a legitimate attack, i.e., $|\delta_t|\le\Delta$ under attack policy $\phi^{sas}_{\Delta_3}$. By Lemma 7 in~\cite{ma2019policy}, we have
	\begin{eqnarray}
	|\delta_t|&=&|R^\prime(s_t,a_t)-R(s_t,a_t)|\\
	&\le&\max_{s,a}[R^\prime(s,a)-R(s,a)]=\|R^\prime-R\|_\infty\\
	&\le& (1+\gamma)\|Q^\prime-Q^*\|=(1+\gamma)\frac{\Delta}{(1+\gamma)}=\Delta.
	\end{eqnarray}
	Therefore the attack policy $\phi^{sas}_{\Delta_3}$ is valid.
\end{proof}

\paragraph{Discussion on a number of non-adaptive attacks:}
Here, we discuss and contrast 3 non-adaptive attack polices developed in this and prior work:
\begin{enumerate}[leftmargin=*, nolistsep]
	\item \cite{huang2019deceptive} produces the non-adaptive attack that is feasible with the smallest $\Delta$. In particular, it solves for the following optimization problem:
	\begin{align}
	\min_{\delta, Q\in \R^{S\times A}}& \|\delta\|_\infty\\
	\mbox{s.t. }& Q(s,a) = \delta(s,a) + \E{P(s'|s,a)}{R(s,a,s) + \gamma \max_{a' \in A} Q(s',a')}\\
	& Q \in \mathcal{Q}^\dagger
	\end{align}
	where the optimal objective value implicitly defines a $\Delta_3'<\Delta_3$. However, it's a fixed policy independent of the actual $\Delta$ . In other word, It's either feasible if $\Delta>\Delta_3'$, or not.
	
	\item $\phi^{sas}_{\Delta_3}$ is a closed-form non-adaptive attack that depends on $\Delta$. $\phi^{sas}_{\Delta_3}$ is guaranteed to be feasible when $\Delta>\Delta_3$. However, this is sufficient but not necessary. Implicitly, there exists a $\Delta_3''$ which is the necessary condition for the feasibility of $\phi^{sas}_{\Delta_3}$. Then, we know $\Delta_3'' > \Delta_3'$, because $\Delta_3'$ is the sufficient and necessary condition for the feasibility of any non-adaptive attacks, whereas $\Delta_3''$ is the condition for the feasibility of non-adaptive attacks of the specific form constructed above.
	
	\item $\phi^{sas}_{TD3}$ (assume perfect optimization) produces the most efficient non-adaptive attack that depends on $\Delta$.
\end{enumerate}
	In terms of efficiency, $\phi^{sas}_{TD3}$ achieves smaller $J_\infty(\phi)$ than $\phi^{sas}_{\Delta_3}$ and \cite{huang2019deceptive}. It's not clear between $\phi^{sas}_{\Delta_3}$ and \cite{huang2019deceptive} which one is better. We believe that in most cases, especially when $\Delta$ is large and learning rate $\alpha_t$ is small, $\phi^{sas}_{\Delta_3}$ will be faster, because it takes advantage of that large $\Delta$, whereas \cite{huang2019deceptive} does not. But there probably exist counterexamples on which \cite{huang2019deceptive} is faster than $\phi^{sas}_{\Delta_3}$.
	
\section{The Covering Time $L$ is $O(\exp(|S|))$ for the chain MDP}
\label{sec:coveringtimeproof}
\begin{proof}
	While the $\epsilon$-greedy exploration policy constantly change according to the agent's current policy $\pi_t$, since $L$ is a uniform upper bound over the whole sequence, and we know that $\pi_t$ will eventually converge to $\pi^\dagger$, it suffice to show that the covering time under $\pi^\dagger_\epsilon$ is $O(\exp(|S|))$.
	
	Recall that $\pi^\dagger$ prefers going right in all but the left most grid. The covering time in this case is equivalent to the expected number of steps taken for the agent to get from $s_0$ to the left-most grid, because to get there, the agent necessarily visited all states along the way. Denote the non-absorbing states from right to left as $s_0, s_1, ..., s_{n-1}$, with $|S| = n$. Denote $V_k$ the expected steps to get from state $s_k$ to $s_{n-1}$. Then, we have the following recursive relation:
	\begin{eqnarray}
	V_{n-1} &=& 0\\
	V_{k} &=& 1 + (1-\frac{\epsilon}{2}) V_{k-1} + \frac{\epsilon}{2} V_{k+1}, \mbox{for } k = 1,...,n-2\\
	V_{0} &=& 1 + (1-\frac{\epsilon}{2}) V_{0} + \frac{\epsilon}{2} V_{1}
	\end{eqnarray}
	Solving the recursive gives
	\begin{equation}
	V_0 = \frac{p(1+p(1-2p))}{(1-2p)^2}\left[(\frac{1-p}{p})^{n-1}-1\right]
	\end{equation}
	where $p = \frac{\epsilon}{2}<\frac{1}{2}$ and thus $V_0 = O(\exp(n))$.
\end{proof}

\section{Proof of Theorem \ref{delta4}}
\label{sec:delta4proof}
\begin{lemma} For any state $s\in S$ and target actions $A(s)\subset A$, it takes FAA at most $\frac{|A|}{1-\epsilon}$ visits to $s$ in expectation to enforce the target actions $A(s)$.
\end{lemma}
\begin{proof}
	Denote $V_t$ the expected number of visits $s$ to teach $A(s)$ given that under the current $Q_t$, $\max_{a \in A(s)}$ is ranked $t$ among all actions, where $t\in 1,...,|A|$. Then, we can write down the following recursion:
	\begin{eqnarray}
	V_1 &=& 0\\
	V_t &=& 1 + (1-\epsilon) V_{t-1} \epsilon \left[\frac{t-1}{|A|}V_{t-1} + \frac{1}{A}V_1 + \frac{|A|-t}{|A|}V_t\right]\label{eq:lemma2}
	\end{eqnarray}
	Equation \eqref{eq:lemma2} can be simplified to 
	\begin{eqnarray}
		V_t &=& \frac{1-\epsilon+\epsilon\frac{t-1}{|A|}}{1-\epsilon\frac{|A|-t}{|A|}}V_{t-1} + \frac{1}{1-\epsilon\frac{|A|-t}{|A|}}\\
		&\leq& V_{t-1}+ \frac{1}{1-\epsilon}
	\end{eqnarray}
	Thus, we have
	\begin{equation}
	V_t\leq \frac{t-1}{1-\epsilon} \leq \frac{|A|}{1-\epsilon}
	\end{equation}
	as needed.
\end{proof}
Now, we prove Theorem \ref{delta4}.
\begin{proof}
	Let $i\in [1,n]$ be given. First, consider the number of episodes, on which the agent was found in at least one state $s_t$ and is equipped with a policy $\pi_t$, s.t. $\pi_t(s_t)\notin \nu_i(s_t)$. Since each of these episodes contains at least one state $s_t$ on which $\nu_i$ has not been successfully taught, and according to Lemma 2, it takes at most $\frac{|A|}{1-\epsilon}$ visits to each state to successfully teach any actions $A(s)$, there will be at most $\frac{|S||A|}{1-\epsilon}$ such episodes. These episodes take at most $\frac{|S||A|H}{1-\epsilon}$ iterations for all target states. Out of these episodes, we can safely assume that the agent has successfully picked up $\nu_i$ for all the states visited.
	
	Next, we want to show that the expected number of iterations taken by $\pi^\dagger_i$ to get to $s_i$ is upper bounded by $\left[\frac{|A|}{\epsilon}\right]^{i-1} D$, where $\pi_i^\dagger$ is defined as 
	\begin{equation}
	\pi_i^\dagger = \argmin_{\pi\in \Pi, \pi(s_j)\in \pi^\dagger(s_j), \forall j\leq i-1} \E{s_0\sim \mu_0}{d_\pi(s_0,s_i)}.
	\end{equation}
	First, we define another policy
	\begin{equation}
	\hat\pi_i^\dagger(s) =
	\left\{
	\begin{array}{ll}
	\pi^\dagger(s)&\mbox{if } s\in\{s_1, ..., s_{i-1}\}\\
	\pi_{s_i}(s)& \mbox{otherwise}
	\end{array}
	\right.
	\end{equation}
	Clearly $\E{s_0\sim \mu_0}{d_{\pi_i^\dagger}(s_0,s_i)} \leq \E{s_0\sim \mu_0}{d_{\hat\pi_i^\dagger}(s_0,s_i)}$ for all $i$.
	
	We now prove by induction that $d_{\hat\pi_i^\dagger}(s,s_i)\leq \left[\frac{|A|}{\epsilon}\right]^{i-1} D$ for all $i$ and $s\in S$.
	
	First, let $i=1$, $\hat\pi_i^\dagger = \pi_{s_1}$, and thus $d_{\hat\pi_i^\dagger}(s,s_i)\leq D$.
	
	Next, we assume that when $i=k$, $d_{\hat\pi_i^\dagger}(s,s_i)\leq D_k$, and would like to show that when $i=k+1$, $d_{\hat\pi_i^\dagger}(s,s_i)\leq \left[\frac{|A|}{\epsilon}\right] D_k$. Define another policy
	\begin{equation}
	\tilde\pi_i^\dagger(s) =
	\left\{
	\begin{array}{ll}
	\pi^\dagger(s)&\mbox{if } s\in\{s_2, ..., s_{i-1}\}\\
	\pi_{s_i}(s)& \mbox{otherwise}
	\end{array}
	\right.
	\end{equation}
	which respect the target policies on $s_2, ..., s_{i-1}$, but ignore the target policy on $s_1$. By the inductive hypothesis, we have that $d_{\tilde\pi_i^\dagger}(s,s_i)\leq D_k$. Consider the difference between $d_{\hat\pi_i^\dagger(s)}(s_1, s_k)$ and $d_{\tilde\pi_i^\dagger}(s_1, s_k)$. Since $\hat\pi_i^\dagger(s)$ and $\tilde\pi_i^\dagger$ only differs by their first action at $s_1$, we can derive Bellman's equation on each policy, which yield
	\begin{eqnarray}
		d_{\hat\pi_i^\dagger}(s_1, s_k) &=& (1-\epsilon)Q(s_1, \pi^\dagger(s_1)) + \epsilon \bar Q(s_1,a)\\
		&\leq& \max_{a \in A} Q(s_1,a)\\
		d_{\tilde\pi_i^\dagger}(s_1, s_k) &=& (1-\epsilon)Q(s_1, \pi_{s_1}(s_1)) + \epsilon \bar Q(s_1,a)\\
		&\geq& \frac{\epsilon}{|A|}\max_{a \in A} Q(s_1,a)\\
	\end{eqnarray}
	where $Q(s_1,a)$ denotes the expected distance to $s_k$ from $s_1$ by performing action $a$ in the first step, and follow $\hat\pi_i^\dagger$ thereafter, and $\bar Q(s_1,a)$ denote the expected distance by performing a uniformly random action in the first step.
	Thus, 
	\begin{equation}
	d_{\hat\pi_i^\dagger}(s, s_k) \leq \frac{|A|}{\epsilon} d_{\tilde\pi_i^\dagger}(s_1, s_k) 
	\end{equation}
	With this, we can perform the following decomposition:
	\begin{eqnarray*}
		d_{\hat\pi_i^\dagger}(s, s_k) &=& \Pr\left[\mbox{visit } s_1 \mbox{ before reaching } s_k\right]\left(d_{\hat\pi_i^\dagger}(s, s_1) + d_{\hat\pi_i^\dagger}(s_1, s_k)\right)+ \Pr\left[\mbox{not visit } s_1\right]\left(d_{\hat\pi_i^\dagger}(s, s_1)\vert\mbox{not visit }s_1\right)\\
		&\leq&  \Pr\left[\mbox{visit } s_1 \mbox{ before reaching } s_k\right]\left(d_{\tilde\pi_i^\dagger}(s, s_1) +\frac{|A|}{\epsilon} d_{\tilde\pi_i^\dagger}(s_1, s_k)\right)+ \Pr\left[\mbox{not visit } s_1\right]\left(d_{\tilde\pi_i^\dagger}(s, s_k)\vert\mbox{not visit }s_1\right)\\
		&=& d_{\tilde\pi_i^\dagger}(s, s_k) + \left(\frac{|A|}{\epsilon}-1\right)d_{\tilde\pi_i^\dagger}(s_1, s_k)\\
		&\leq& D_k + \left(\frac{|A|}{\epsilon}-1\right) D_k = \frac{|A|}{\epsilon}D_k.
	\end{eqnarray*}
	This completes the induction. Thus, we have 
	\begin{equation}
	d_{\hat\pi_i^\dagger}(s, s_i) \leq \left(\frac{|A|}{\epsilon}\right)^{i-1}D,
	\end{equation}
	and the total number of iterations taken to arrive at all target states sequentially sums up to
	\begin{equation}
	\sum_{i=1}^n d_{\hat\pi_i^\dagger}(s, s_i) \leq \left(\frac{|A|}{\epsilon}\right)^{n}D.
	\end{equation}
	Finally, each target states need to visited for $\frac{|A|}{1-\epsilon}$ number of times to successfully enforce $\pi^\dagger$. Adding the numbers for enforcing each $\pi^\dagger_i$ gives the correct result.
\end{proof}

\section{Detailed Explanation of Fast Adaptive Attack Algorithm}
\label{sec:FAA}
In this section, we try to give a detailed walk-through of the Fast Adaptive Attack Algorithm (FAA) with the goal of providing intuitive understanding of the design principles behind FAA. For the sake of simplisity, in this section we assume that the Q-learning agent is $\epsilon = 0$, such that the attacker is able to fully control the agent's behavior. The proof of correctness and sufficiency in the general case when $\epsilon \in [0,1]$ is provided in section \ref{sec:delta4proof}.

\paragraph{The Greedy Attack:}To begin with, let's talk about \emph{the greedy attack}, a fundamental subroutine that is called in every step of FAA to generate the actual attack. Given a desired (partial) policy $\nu$, the greedy attack aims to teach $\nu$ to the agent in a greedy fashion. Specifically, at time step $t$, when the agent performs action $a_t$ at state $s_t$, the greedy attack first look at whether $a_t$ is a desired action at $s+t$ according to $s\nu$, i.e. whether $a_t\in \nu(s_t)$. If $a_t$ is a desired action, the greedy attack will produce a large enough $\delta_t$, such that after the Q-learning update, $a_t$ becomes strictly more preferred than all undesired actions, i.e. $Q_{t+1}(s_t,a_t)> \max_{a\notin \nu(s_t)}Q_{t+1}(s_t,a)$. On the other hand, if $a_t$ is not a desired action, the greedy attack will produce a negative enough $\delta_t$, such that after the Q-learning update, $a_t$ becomes strictly less preferred than all desired actions, i.e. $Q_{t+1}(s_t,a_t)< \max_{a\in \nu(s_t)}Q_{t+1}(s_t,a)$. It can be shown that with $\epsilon = 0$, it takes the agent at most $|A|-1$ visit to a state $s$, to force the desired actions $\nu(s)$.

Given the greedy attack procedure, one could directly apply the greedy attack with respect to $\pi^\dagger$ throughout the attack procedure. The problem, however, is efficiency. The attack is not considered success without the attacker achieving the target actions in ALL target states, not just the target states visited by the agent. If a target state is never visited by the agent, the attack never succeed. $\pi^\dagger$ itself may not efficiently lead the agent to all the target states. A good example is the chain MDP used as the running example in the main paper. In section \ref{sec:coveringtimeproof}, we have shown that if an agent follows $\pi^\dagger$, it will take exponentially steps to reach the left-most state. In fact, if $\epsilon = 0$, the agent will never reach the left-most state following $\pi^\dagger$, which implies that the naive greedy attack w.r.t. $\pi^\dagger$ is in fact infeasible. Therefore, explicit navigation is necessary. This bring us to the second component of FAA, \emph{the navigation polices}.

\paragraph{The navigation polices:} Instead of trying to achieve all target actions at once by directly appling the greedy attack w.r.t. $\pi^\dagger$, FAA aims at one target state at a time. Let $s^\dagger_{(1)}, ..., s^\dagger_{(k)}$ be an order of target states. We will discuss the choice of ordering in the next paragraph, but for now, we will assume that an ordering is given. The agent starts off aiming at forcing the target actions in a single target state $s^\dagger_{(1)}$. To do so, the attacer first calculate the corresponding navigation policy $\nu_1$, where $\nu_1(s_t) = \pi_{s^\dagger_{(1)}}(s_t)$ when $s_t\neq s^\dagger_{(1)}$, and $\nu_1(s_t) = \pi^\dagger(s_t)$ when $s_t= s^\dagger_{(1)}$. That is, $\nu_1$ follows the shortest path policy w.r.t. $s^\dagger_{(1)}$ when the agent has not arrived at $s^\dagger_{(1)}$, And when the agent is in $s^\dagger_{(1)}$, $\nu_1$ follows the desired target actions. Using the greedy attack w.r.t. $\nu_1$ allows the attacker to effectively lure the agent into $s^\dagger_{(1)}$ and force the target actions $\pi^\dagger(s^\dagger_{(1)})$. After successfully forcing the target actions in $s^\dagger_{(1)}$, the attacker moves on to $s^\dagger_{(2)}$. This time, the attacker defines the navigation policy $\nu_2$ similiar to $\nu_1$, except that we don't want the already forced $\pi^\dagger(s^\dagger_{(1)})$ to be untaught. As a result, in $\nu_2$, we define $\nu_2(s^\dagger_{(1)}) = \pi^\dagger(s^\dagger_{(1)})$, but otherwise follows the corresponding shortest-path policy $\pi_{s^\dagger_{(2)}}$. Follow the greedy attack w.r.t. $\nu_2$, the attacker is able to achieve $\pi^\dagger(s^\dagger_{(2)})$ efficiently without affecting $\pi^\dagger(s^\dagger_{(1)})$. This process is carried on throughout the whole ordered list of target states, where the target actions for already achieved target states are always respected when defining the next $\nu_i$. If each target states $s^\dagger_{(i)}$ can be reachable with the corresponding $\nu_i$, then the whole process will terminate at which point all target actions are guaranteed to be achieved. However, the reachability is not always guaranteed with any ordering of target states. Take the chain MDP as an example. if the 2nd left target state is ordered before the left-most state, then after teaching the target action for the 2nd left state, which is moving right, it's impossible to arrive at the left-most state when the navigation policy resepct the moving-right action in the 2nd left state. Therefore, the \emph{ordering} of target states matters.

\paragraph{The ordering of target states:} FAA orders the target states descendingly by their shortest distance to the starting state $s_0$. Under such an ordering, the target states achieved first are those that are farther away from the starting state, and they necessarily do not lie on the shortest path of the target states later in the sequence. In the chain MDP example, the target states are ordered from left to right. This way, the agent is always able to get to the currently focused target state from the starting state $s_0$, without worrying about violating the already achieved target states to the left. However, note that the bound provided in theorem \ref{delta4} do not utilize this particular ordering choice and applies to any ordering of target states. As a result, the bound diverges when $\epsilon\rightarrow 0$, matching with the pathological case described at the end of the last paragraph.

\section{Experiment Setting and Hyperparameters for TD3}
\label{sec:TD3params}
Throughout the experiments, we use the following set of hyperparameters for TD3, described in Table \ref{table:hyperparameters}. The hyperparameters are selected via grid search on the Chain MDP of length 6. Each experiment is run for 5000 episodes, where each episode is of 1000 iteration long. The learned policy is evaluated for every $10$ episodes, and the policy with the best evaluation performance is used for e evaluations in the experiment section.
\begin{table}[t]
	\centering
	\begin{tabular}{| l | c | l |} 
		\hline
		Parameters & Values & Description\\ [0.5ex] 
		\hline\hline
		exploration noise & $0.5$ & Std of Gaussian exploration noise.\\
		batch size & 100 & Batch size for both actor and critic\\
		discount factor & 0.99 & Discounting factor for the attacker problem.\\
		policy noise & 0.2 & Noise added to target policy during critic update.\\
		noise clip & $[-0.5, 0.5]$ & Range to clip target policy noise.\\
		action L2 weight & 50 & Weight for L2 regularization added to the actor network optimization objective.\\
		buffer size & $10^7$ & Replay buffer size, larger than total number of iterations.\\
		optimizer & Adam & Use the Adam optimizer.\\
		learning rate critic & $10^{-3}$ & Learning rate for the critic network.\\
		learning rate actor & $5^{-4}$ & Learning rate for the actor network.\\
		$\tau$ & $0.002$ & Target network update rate.\\
		policy frequency & 2& Frequency of delayed policy update.\\
		\hline
	\end{tabular}
	\caption{Hyperparameters for TD3.}
	\label{table:hyperparameters}
\end{table}

\section{Additional Experiments}
\subsection{Additional Plot for the rate comparison experiment}
\label{sec:normal scale}
See Figure \ref{fig:exp_length_normal}.
\begin{figure}[ht!]
	\centering
	\includegraphics[width=0.5\columnwidth]{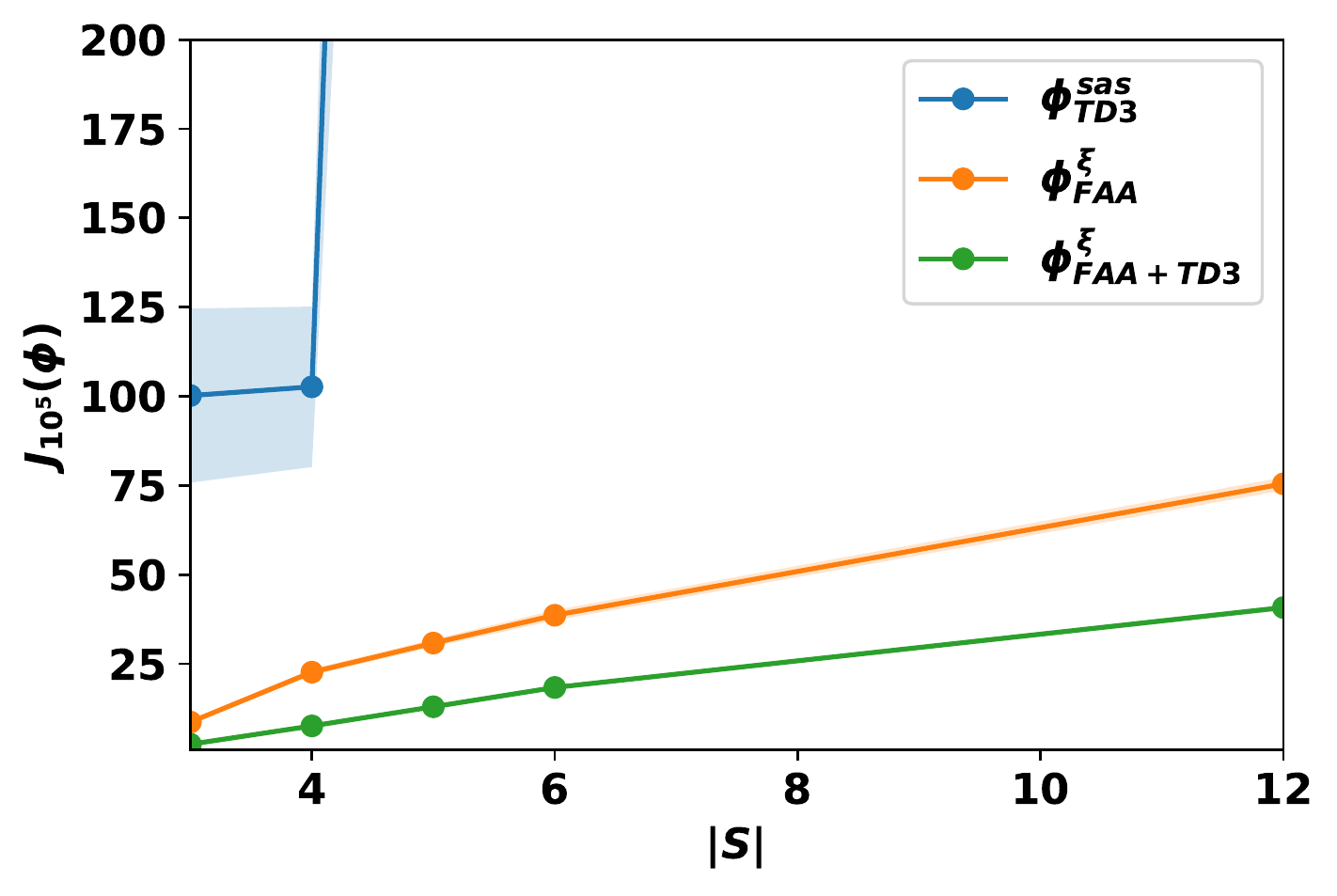}
	\caption{Attack performances on the chain MDP of different length in the normal scale. As can be seen in the plot, both $\phi^{\xi}_{FAA}$ + $\phi^{\xi}_{TD3 + FAA}$ achieve linear rate.}
	\label{fig:exp_length_normal}
\end{figure}

\subsection{Additional Experiments: Attacking DQN}
\label{sec:DQN}

Throughout the main paper, we have been focusing on attacking the tabular Q-learning agent. However, the attack MDP also applies to arbitrary RL agents. We describe the general interaction protocol in Alg. \ref{alg:protocol2}. Importantly, we assume that the RL agent can be fully characterized by an \textbf{internal state}, which determines the agent's current behavior policy as well as the learning update.
\begin{algorithm}[ht!]
	\caption{Reward Poisoning against general RL agent}\label{alg:protocol2}
	\begin{flushleft}
		\textbf{Parameters:} MDP $(S, A, R, P, \mu_0)$, RL agent hyperparameters.\\
	\end{flushleft}
	\begin{algorithmic}[1]
		\FOR{$t = 0,1,...$}
		\STATE agent at state $s_t$, has internal state $\theta_0$.
		\STATE agent acts according to a behavior policy:\\
		$a_t \leftarrow \pi_{\theta_t}(s_t)$
		\STATE environment transits $s_{t+1} \sim P(\cdot \mid s_t, a_t)$, produces reward $r_t=R(s_t,a_t,s_{t+1})$ and an end-of-episode indicator $EOE$.
		\STATE attacker perturbs the reward to $r_t+\delta_t$
		\STATE agent receives $(s_{t+1}, r_t+\delta_t, EOE)$,  performs one-step of internal state update:
		\begin{eqnarray}
		\theta_{t+1} = f(\theta_t, s_t, a_t, s_{t+1}, r_t+\delta_t, EOE)
		\end{eqnarray}
		\STATE environment resets if $EOE = 1$: $s_{t+1} \sim \mu_0$.
		\ENDFOR
	\end{algorithmic}
\end{algorithm}
\begin{figure}[ht!]
	\centering
	\includegraphics[width=0.5\columnwidth]{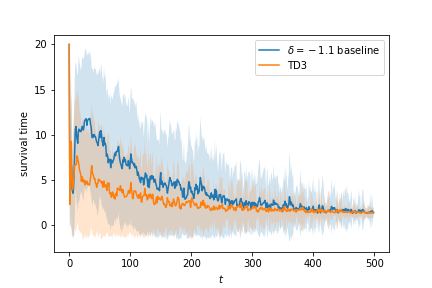}
	\caption{Result for attacking DQN on the Cartpole environment. The left figure plots the cumulative attack cost $J_T(\phi)$ as a function of $T$. The right figure plot the performance of the DQN agent $J(\theta_t)$ under the two attacks.}
	\label{fig:cartpole}
\end{figure}
For example, if the RL agent is a Deep Q-Network (DQN), the internal state will consist of the Q-network parameters as well as the transitions stored in the replay buffer.

In the next example, we demonstrate an attack against DQN in the cartpole environment. In the cartpole environment, the agent can perform 2 actions, moving left and moving right, and the goal is to keep the pole upright without moving the cart out of the left and right boundary. The agent receives a constant $+1$ reward in every iteration, until the pole falls or the cart moves out of the boundary, which terminates the current episode and the cart and pole positions are reset.

In this example, the attacker's goal is to poison a well-trained DQN agent to perform as poorly as possible. The corresponding attack cost $\rho(\xi_t)$ is defined as $J(\theta_t)$, the expected total reward received by the current DQN policy in evaluation. The DQN is first trained in the clean cartpole MDP and obtains the optimal policy that successfully maintains the pole upright for 200 iterations (set maximum length of an episode). The attacker is then introduced while the DQN agent continues to train in the cartpole MDP. We freeze the Q-network except for the last layer to reduce the size of the attack state representation. We compare TD3 with a naive attacker that perform $\delta_t = -1.1$ constantly. The results are shown in Fig. \ref{fig:cartpole}. 

One can see that under the TD3 found attack policy, the performance of the DQN agent degenerates much faster compared to the naive baseline. While still being a relatively simple example, this experiment demonstrates the potential of applying our adaptive attack framework to general RL agents.
\end{document}

%% file: def.tex
\newcommand{\Q}{\mathcal{Q}}

\newtheorem{theorem}{Theorem}
\newtheorem{lemma}[theorem]{Lemma}

\newtheorem{corollary}[theorem]{Corollary}
\newtheorem{definition}{Definition}

\newcommand{\thmref}[1]{Theorem~\ref{#1}}

\renewcommand\ind[1]{\ensuremath{\mathds{1}\left[#1\right]}}
\renewcommand{\E}[2]{\mathbb{E}_{#1}\left[#2\right]}
\newcommand{\interior}[1]{%
  {\kern0pt#1}^{\mathrm{o}}%
}